\documentclass[preprint,onefignum,onetabnum]{siamonline220329}

\usepackage{lipsum}
\usepackage{amsfonts}
\usepackage{graphicx}
\usepackage{amssymb}
\usepackage{epstopdf}
\newtheorem{assumption}[theorem]{Assumption}
\newtheorem{discussion}{Discussion}
\theoremstyle{plain}
\newtheorem{remark}[theorem]{Remark}
\newtheorem{condition}{Condition}

\newcommand{\bX}{\mathsf{X}}
\newcommand{\bU}{\mathsf{U}}
\newcommand{\bY}{\mathsf{Y}}
\newcommand{\bH}{\mathsf{H}}
\newcommand{\bZ}{\mathsf{Z}}

\theoremstyle{empty}
\newtheorem{duplicate}{NameIgnored}

\newcommand{\cO}{\mathcal{O}}
\newcommand{\cK}{\mathcal{K}}
\newcommand{\cL}{\mathcal{L}}
\newcommand{\cA}{\mathcal{A}}
\newcommand{\cV}{\mathcal{V}}
\newcommand{\cQ}{\mathcal{Q}}
\newcommand{\cP}{\mathcal{P}}
\newcommand{\cB}{\mathcal{B}}
\newcommand{\bE}{\mathbb{E}}
\newcommand{\bP}{\mathbb{P}}
\newcommand{\bR}{\mathbb{R}}

\newcommand{\defn}{\stackrel{\rm def}{=}}

\usepackage{algorithmic}
\ifpdf
  \DeclareGraphicsExtensions{.eps,.pdf,.png,.jpg}
\else
  \DeclareGraphicsExtensions{.eps}
\fi

\usepackage{enumitem}
\setlist[enumerate]{leftmargin=.5in}
\setlist[itemize]{leftmargin=.5in}

\newsiamremark{hypothesis}{Hypothesis}
\crefname{hypothesis}{Hypothesis}{Hypotheses}
\newsiamthm{claim}{Claim}

\headers{Finite-Time Analysis of NAC for POMDPs}{S. Cayci, Niao He, R. Srikant}

\title{Finite-Time Analysis of Natural Actor-Critic for POMDPs\footnote{This work was supported by Illinois Institute for Data Science and Dynamical Systems funded by NSF Award \#1934986.}}

\author{Semih Cayci\thanks{Department of Mathematics, RWTH Aachen University, Aachen, Germany 
  (\email{cayci@mathc.rwth-aachen.de}).}
\and Niao He\thanks{Department of Computer Science, ETH Zurich, Zurich, Switzerland 
  (\email{niao.he@inf.ethz.ch}).}
\and R. Srikant\thanks{Department of Electrical and Computer Engineering, University of Illinois at Urbana-Champaign, Urbana, IL 
  (\email{rsrikant@illinois.edu}).}}

\usepackage{amsopn}

\ifpdf
\hypersetup{
  pdftitle={Finite-Time Analysis of Natural Actor-Critic for POMDPs},
  pdfauthor={S. Cayci, N. He, and R. Srikant}
}
\fi

\begin{document}

\maketitle

\begin{abstract}
We study the reinforcement learning problem for partially observed Markov decision processes (POMDPs) with large state spaces. We consider a natural actor-critic method that employs an internal memory state for policy parameterization to address partial observability, function approximation in both actor and critic to address the curse of dimensionality, and a multi-step temporal difference learning algorithm for policy evaluation. We establish non-asymptotic error bounds for actor-critic methods for partially observed systems under function approximation. In particular, in addition to the function approximation and statistical errors that also arise in MDPs, we explicitly characterize the error due to the use of finite-state controllers. This additional error is stated in terms of the total variation distance between the belief state in POMDPs and the posterior distribution of the hidden state when using a finite-state controller. Further, in the specific case of sliding-window controllers, we show that this inference error can be made arbitrarily small by using larger window sizes under certain ergodicity conditions.
\end{abstract}



\section{Introduction}
\label{sec:intro}

The class of optimal control problems where the controller has access to only noisy observations of the system state is modeled as partially observed Markov decision processes (POMDPs) \cite{sondik1971optimal, smallwood1973optimal, kumar2015stochastic, bertsekas2012dynamic}. Since the underlying state is only partially known to the controller in POMDPs, the optimal policy depends on the complete history of the system, making the problem highly intractable \cite{bertsekas2012dynamic, krishnamurthy2016partially}. To overcome computational challenges in solving POMDPs, a plethora of model-based and model-free reinforcement learning approaches have been proposed in the literature that incorporate finite memory into the controller, via internal state or quantization of belief state; see, e.g., surveys by \cite{shani2013survey, krishnamurthy2016partially, murphy2000survey}. 

The actor-critic framework, which combines the benefits of both value-based methods and policy gradient methods, has shown great promise in learning POMDPs in practice~\cite{wierstra2010recurrent, yu2006approximate,srinivasan2018actor,jurvcivcek2011natural,xu2015acis, lee2020stochastic}. These methods offer more flexibility in controlling the bias-variance tradeoff. Furthermore, in the case of POMDPs, finite-state stochastic policies, which we aim to learn in this paper by using the NAC framework, were shown to achieve superior practical performance \cite{williams1998experimental}. However, theoretical analyses of the convergence rates and optimality properties of these POMDP solvers seem largely absent, particularly in the interesting case of function approximation for large state-action-observation spaces. In this paper, we provide new results for this fundamental problem.

\subsection{Main Contributions}
In this paper, we consider a model-free natural actor-critic (NAC) method for POMDPs: (i) the actor employs an internal state representation as a form of memory and performs efficient natural policy gradient update; and (ii) the critic employs a multi-step temporal difference learning algorithm to obtain the value functions.
Our main contributions include the following:

\textbullet~\textbf{Finite-time analysis of NAC for POMDPs.} We establish, to our best knowledge, the first finite-time performance bounds of actor-critic-type methods with function approximation for large POMDPs, with explicit characterization of the convergence rate, function approximation error, and inference error due to the partial observability.

\textbullet~\textbf{Multi-step TD(0) for policy evaluation for POMDPs.}~We consider a multi-step TD learning algorithm with linear function approximation to learn the value function under an internal-state controller, and establish finite-time bounds for this algorithm. In particular, in Theorem \ref{thm:m-tdl}, we show that policy evaluation with $m$-step TD(0) learning resolves the perceptual aliasing error that stems from partial observability at a rate $\cO(\gamma^{m/2})$ at the expense of an extra factor $\cO(m)$ in sample complexity, which identifies a tradeoff between sample complexity and perceptual aliasing error.

\textbullet~\textbf{Memory-Inference Error Tradeoff under Sliding-Window Controllers.}~We further consider NAC with sliding-window controllers as a specific case of finite-state controllers and provide explicit bounds on the tradeoff between memory complexity and inference error in Proposition \ref{prop:filter-stability}. Notably, under ergodicity conditions to ensure filter stability, the inference error (due to using limited memory) decays at a geometric rate in the window-length.

\subsection{Related Work}

\textbf{PG/NPG/NAC for MDPs:} Policy gradient methods have been extensively investigated for fully-observed MDPs \cite{konda2003onactor, agarwal2021theory, bhandari2019global, lan2021policy, mei2020global, xu2020improving, khodadadian2021finite}. As they rely on the perfect state observation in MDPs, they do not address the problem of partial observability that we consider in this paper.

\textbf{Policy Evaluation for POMDPs:} Policy evaluation methods have been considered in \cite{singh1994learning, baxter2001infinite} for average-reward POMDPs for the specific class of reactive (memoryless) policies in the tabular setting. In \cite{yu2005function}, tabular TD(1) was adapted for finite-state controllers. In our work, as part of the natural actor-critic framework, we present finite-time analysis of multi-step TD learning algorithm with linear function approximation for large POMDPs. For POMDPs, finite-step TD learning methods do not converge to the true value functions for policy evaluation unlike MDPs. However, we show in this paper that the error can be controlled by employing a multi-step TD learning algorithm, which leads to a tradeoff between sample complexity and accuracy. In particular, we prove the effectiveness of multi-step TD learning for POMDPs by providing a non-asymptotic analysis, which yields explicit sample complexity bounds and exhibits the impact of partial observability and function approximation. 

\textbf{RL for POMDPs:} In \cite{kara2021convergence}, tabular Q-learning with sliding-window controllers was considered for the tabular case. In \cite{efroni2022provable, wang2022embed}, the RL problem for POMDPs was investigated under an $m$-step decodability assumption and under linear transition dynamics. In \cite{subramanian2022approximate}, general internal state structures are considered and asymptotic convergence is shown in the tabular case. In this paper, we consider general POMDPs under linear function approximation without any realizability assumptions, and explicitly quantify (i) inference error due to using finite-state controllers, and (ii) function approximation error due to non-linear dynamics. We then show that the subclass of sliding-window controllers asymptotically achieves near-optimality under filter stability conditions.

Our paper takes a different path in terms of RL methodology, and considers a policy-based actor-critic approach that employs a general class of internal state representation. Importantly, we optimize policy over the extended class of finite-state \textit{stochastic} policies, which were observed to outperform deterministic policies for POMDPs \cite{williams1998experimental}. Actor-critic methods were first mentioned as a potential solution method for POMDPs in \cite{murphy2000survey, yu2005function} and empirically studied in several papers, e.g., ~\cite{lee2020stochastic, srinivasan2018actor,jurvcivcek2011natural,xu2015acis}. However, none of these works provides non-asymptotic convergence guarantees or optimality properties.

\subsection{Notation}
For a sequence $(s_k)_{k\in\mathbb{N}}$ over a set $\mathsf{S}$, the vector $(s_i,s_{i+1},\ldots,s_j)$ for any $i \leq j$ is denoted by $s_i^j$. Let $s^n$ denote $s_0^n$. We denote the cardinality of a finite set $\mathsf{S}$ by $|\mathsf{S}|$. For a countable set $\mathsf{S}$, $\Sigma(\mathsf{S})$ denotes the simplex over $\mathsf{S}$:
    $\Sigma(\mathsf{S}) \defn \{v\in\bR_+^{|\mathsf{S}|}: \sum_{i\in\mathsf{S}}v_i = 1\}.$
For $\xi\in\Sigma(\mathsf{S})$, $\mathsf{supp}(\xi) \defn \{s\in\mathsf{S}:\xi(s) > 0\}$ denotes the support set of $\xi$. For $v\in \bR^{|\mathsf{S}|}$ and $\xi\in\Sigma(\mathsf{S})$, we denote the weighted-$\ell_2$ norm as $\|v\|_\xi\defn\sqrt{\sum_{i\in\mathsf{S}}\xi(i)\left|v(i)\right|^2}.$ For any finite set $\bX$, $\|P-Q\|_{\rm TV}\defn\frac{1}{2}\sum_{x\in\bX}|P(x)-Q(x)|$ denotes the total variation distance, and $\mathfrak{D}_{\rm KL}(P\|Q)$ denotes Kullback-Leibler divergence between two distributions $P,Q\in\Sigma(\bX)$. $\mathcal{B}_2(x, R) \defn \{z\in\bR^d:\|x-z\|_2 \leq R\}$ denotes the $\ell_2$-ball with center $x\in\bR^d$ and radius $R$. For any $w_0\in\bR^d$, $\mathbf{Proj}_{\mathcal{C}}(w_0) = \arg\min_{w\in\mathcal{C}}\|w-w_0\|_2,$ denotes the projection of $w_0$ onto the subset $\mathcal{C}\subset\bR^d$. For a matrix $A\in\bR^{m\times n}$, $A^\dagger$ denotes its Moore-Penrose inverse. $\sigma(X_\theta:\theta\in\Theta)$ denotes the $\sigma$-field generated by a collection of random variables $\{X_\theta:\theta\in\Theta\}$ for an index set $\Theta$.

\section{POMDPs and Finite-State Controllers}\label{sec:pomdp}
We consider a discrete-time dynamical system with an finite but arbitrarily large state space $\bX$, and finite control space $\bU$. $\{(X_k,U_k)\in\bX\times\bU:k\in\mathbb{N}\}$ is a time-homogenous controlled Markov chain, which evolves according to $$\bP(X_{k+1}=x'|X_k=x,U_k=u) = \mathcal{P}(x'|x,u),$$ for any $k\geq 0,~x,x'\in\bX$ and $u\in\bU$,
where $\mathcal{P}$ is the transition kernel. The system state $\{X_k:k\in\mathbb{N}\}$ is available to the controller only through a (noisy) discrete memoryless observation channel:
    $\bP(Y_k=y|X_k=x) = \Phi(y|x),k\in\mathbb{N}, (x,y)\in\bX\times\bY$,
where $Y_k\in\bY$ is the observation and $\Phi(\cdot|x)\in\Sigma(\bY)$ is the observation channel for any $x\in\bX$. The channel is memoryless in the following sense: $$\bP(Y^k=y^k|X^k=x^k) = \prod_{i=0}^k\Phi(y_i|x_i),$$ for any $(x^k,y^k)\in\bX^{k+1}\times\bY^{k+1}$ and $k\geq 0$. The information available to the controller at time $k$ is $H_k = (H_{k-1}, Y_k, U_{k-1})$ with initial $H_0\in\bH$. An admissible policy $\pi = (\mu_0,\mu_1,\ldots)$ is a sequence of mappings $\mu_k:\bH\times\bY^{k}\times\bU^k\rightarrow \Sigma(\bU), k \geq 0$. Applying control $u\in\bU$ at state $x\in\bX$ yields a reward $r(x,u)\in[0,1]$.

\subsection{Value Functions for POMDPs}
For a given admissible policy $\pi\in\Pi_A$, the value function is defined as the expected $\gamma$-discounted total reward given the initial knowledge $h_0\in\bH$: 
\begin{equation}
\mathcal{V}^{\pi}(h_0) = \bE\Big[\sum_{k=0}^\infty \gamma^k r(X_k,U_k)\Big|H_0=h_0\Big].
    \label{eqn:filtered-v-function}
\end{equation}
Similarly, we define the Q-function under $\pi$ as follows: $$\mathcal{Q}^\pi(h_0,u_0) = \bE\Big[\sum_{k=0}^\infty\gamma^k r(X_k,U_k)\Big|H_0=h_0, U_0=u_o\Big],$$ for any $~h_0\in\bH,u_0\in\bU.$ The advantage function is defined as
\begin{equation*}
\cA^\pi(h_0,u_0) = \cQ^\pi(h_0,u_0) - \cV^\pi(h_0),~(h_0,u_0)\in\bH\times\bU.
\end{equation*}

The ultimate objective is to find the optimal policy over the class of admissible policies that maximizes the discounted reward given an initial distribution  $\xi\in\Sigma(\mathsf{H})$, namely,
\begin{equation}
   \max_{\pi\in\Pi_A}~ \cV^\pi(\xi) \defn \bE_{H_0\sim\xi} [\cV^\pi(H_0)].
   \label{eqn:pomdp-objective}
\end{equation}
Note that the optimal controller for a POMDP bases its decisions on $H_k \in \bH\times\bY^{k}\times\bU^k$, thus an exponentially growing memory over time is required for policy optimization, which is known as the \emph{curse of history}.

\subsection{Bayes Filtering and Belief State Formulation}\label{subsec:belief}
Let
\begin{equation}
    b_k(x,h) \defn \bP(X_k = x|H_k = h),
\end{equation}
be the belief state at time $k \geq 0$. We denote $b_k(x,h_k)$ as $b_k(x)$ in short, and the belief can be computed in a recursive way by the following filtering transformation:
\begin{align}
    b_k(x) &= \frac{\sum_{x^\prime}b_{k-1}(x^\prime)\mathcal{P}(x|x^\prime,u_{k-1})\Phi(y_k|x)}{\sum_{x^\prime, x''}b_{k-1}(x^\prime)\mathcal{P}(x''|x^\prime,u_{k-1})\Phi(y_k|x'')} \defn F(b_{k-1}, y_k, u_{k-1})(x),
\end{align}
which follows from the Bayes theorem \cite{kumar2015stochastic, krishnamurthy2016partially}. We denote $k\geq 0$ successive applications of the filtering transformation $F$ as follows:
\begin{equation}
    F^{(k)}(b_0, y_1^k, u_0^{k-1})(x) = b_k(x,h_k).
    \label{eqn:filtering}
\end{equation}
For any $u\in\bU$, let $\tilde{r}(b_k, u) = \sum_{x\in\bX}b_k(x, h_k)r(x, u)$. Then, the problem reduces to a fully observable MDP where $(b_k, u_k)$ forms a controlled Markov chain, and action $u_k$ at belief state $b_k$ yields a reward $\tilde{r}(b_k, u)$ \cite{bertsekas2012dynamic}. Therefore, the techniques for MDPs can, in theory, be applied to solve the POMDP problem \eqref{eqn:pomdp-objective}. On the other hand, the belief $b_k$ is continuous-valued, which makes the policy search problem highly intractable. This constitutes the main challenge in RL for POMDPs \cite{murphy2000survey, krishnamurthy2016partially}.

\subsection{Finite-State Controllers for POMDPs}
In order to address the curse of history and to achieve tractability in solving POMDPs, controllers that employ an internal state to summarize the history are widely used \cite{yu2008near, baxter2001infinite, singh1994learning, murphy2000survey}. 
In this paper, we will mainly focus on this subclass of admissible policies and investigate its performance guarantees. 

\begin{definition}[Internal State Representation]\normalfont
An internal state representation is a pair $(\bZ,\varphi)$ where $\bZ$ is a finite set, and $\varphi$ is a transition kernel such that the internal state $\{Z_k:k\geq 0\}$, which keeps a summary of $H_k$, is a stochastic process over $\bZ$ with the following transition:
\begin{equation*}
\bP(Z_{k+1}=z'|Z_k=z,Y_k=y,U_k=u) = \varphi(z'|z, y, u),~\forall (z',z,y,u)\in\bZ^2\times\bY\times\bU.
\end{equation*}
\end{definition}

\begin{definition}[Finite-State Controller]\normalfont
An admissible policy $\pi = (\mu_0, \mu_1, \ldots)$ such that $\mu_k$ bases its decision on the latest observation $Y_k$ and the internal state $Z_k$ for any $k$, i.e., $\mu_k:(Y_k,Z_k)\mapsto U_k$, is a finite-state controller (FSC). The class of FSCs is denoted as $\Pi_{Z,\varphi}$.
\end{definition}

In this specific case, the initial knowledge of the controller about the system, $h_0\in\bH$, is the vector $(y_0,z_0)\in\bY\times\bZ$, thus $\bH=\bY\times\bZ$. The goal in this paper is to learn an optimal FSC for a given internal state representation $(\bZ, \varphi)$. 

\begin{definition}[Optimal FSC]\normalfont
    For a given $(\bZ,\varphi)$ and prior distribution $\xi\in\Sigma(\bH)$, the optimal FSC is defined as
    \begin{equation}
        \pi^* = \arg\max_{\pi\in\Pi_{\bZ,\varphi}}~\cV^\pi(\xi).
    \end{equation}
\end{definition}

\subsection{Sliding-Window Controllers} 
An important subclass of finite-state controllers is sliding-window controllers (SWC) \cite{loch1998using, yu2008near, sung2017learning, kara2021convergence}, which was shown to achieve good practical performance, particularly in combination with stochastic policies (which we aim to learn in this paper by using the NAC framework) \cite{williams1998experimental}. For a given window-length $n > 0$, the internal state is defined as $Z_k = (Y_{k-n}^{k-1}, U_{k-n}^{k-1}) \in \bY^n\times\bU^n.$
For $n = 0$, the internal state is null, thus the controller bases its decisions at time $k$ only on the last observation $Y_k$, which is called a \emph{reactive} or \emph{memoryless} policy \cite{singh1994learning, yu2005function}. For sliding-window controllers, the initial internal state is $H_0 = (Y_0,Z_0) = (Y_{-n}^0,U_{-n}^{-1})\in\bY^{n+1}\times\bU^n$, thus $\bH = \bY^{n+1}\times\bU^n$.

\section{Natural Actor-Critic for POMDPs}\label{sec:critic}
In this section, we develop a natural actor-critic (NAC) framework in order to find the optimal policy within the class of FSCs.

\subsection{Policy Parameterization}
We consider softmax parameterization for FSCs with linear function approximation. Given a feature set $\Psi=\{\psi(u,y,z)\in\bR^d:(u,y,z)\in\bY\times\bZ\times\bU\}$,
\begin{equation}
    \pi_\theta(u|y, z) = \frac{\exp\left(\theta^\top\psi(u, y, z)\right)}{\sum_{u^\prime\in\bU}\exp\left(\theta^\top\psi(u',y,z)\right)},
    \label{eqn:policy-parameterization}
\end{equation}
for all $(u, y, z)\in\bU\times\bY\times\bZ$. Under the observation and internal state pair $(y,z)\in\bY\times\bZ$, the controller makes a randomized decision $u\sim\pi_\theta(\cdot|y,z)$.

\subsection{Sampling}
We define the discounted state-action visitation distribution under $\pi$ as $$\mathbf{d}_{h_0}^\pi(y,z) \defn(1-\gamma) \sum_{k=0}^\infty\gamma^k\bP^\pi(Y_k=y,Z_k=z|H_0=h_0),~\mbox{for}~(y,z),h_0\in\bY\times\bZ,$$ For any initial distribution $\xi\in\Sigma(\bY\times\bZ)$, we denote $\mathbf{d}_\xi^\pi(y,z) = \bE_{H_0\sim \xi}[\textbf{d}_{H_0}^\pi(y,z)],~(y,z)\in\bY\times\bZ$.
\begin{assumption}[Sampling oracle]\normalfont
We assume that the controller is able to obtain an independent sample $H_0\sim\mathbf{d}_\xi^\pi$ at any time.
\label{assumption:sampling-oracle}
\end{assumption}
The sampling procedure is specified in Remark \ref{remark:sampling}, which extends sampling from state-visitation distribution for MDPs \cite{konda2003onactor, agarwal2021theory}.

\subsection{Natural Actor-Critic Algorithm for Finite-State Controllers}
NAC algorithm, summarized in Algorithm \ref{alg:nac}, works as follows. We initialize the policy optimization at the max-entropy policy by setting $\theta_0 = {0}$. At iteration $t \geq 0$, the policy parameter is denoted by $\theta_t$, and the corresponding policy is $\pi_t := \pi_{\theta_t}$. NAC algorithm consists of the following steps:

\textbf{Step 1: (Critic)} Obtain an approximate state-action value function $\widehat{\cQ}_K^{\pi_t}$ by using multi-step TD learning, as described in Algorithm \ref{alg:td}, which is sufficient to compute:
\begin{equation}
    \widehat{\cA}_K^{\pi_t}(u,y,z) = \widehat{\cQ}_K^{\pi_t}(u,y,z) - \widehat{\cV}_K^{\pi_t}(y,z),
\end{equation}
where $\widehat{\cV}_K^{\pi_t}(y,z) = \sum_{u\in\bU}\pi_t(u|y,z)\widehat{\cQ}_K^{\pi_t}(u,y,z),$ for $K$ critic steps per iteration.

\begin{algorithm}[tb]
  \caption{$m$-step TD learning for POMDPs}
  \label{alg:td}
    \begin{algorithmic}
   \STATE {\bfseries Input:} $\pi\in\Pi_{\bZ,\varphi}$ policy, $m$: memory size, $\alpha$: step-size, $K$: time-horizon, $R$: proj. radius\\
   \FOR{$k=0$ {\bfseries to} $K-1$}
        \STATE Sample $h_0(k) = (y_0(k),z_0(k))\sim\mathbf{d}_\xi^{\pi_k}$,\\
        \FOR{$i=0$ to $m-1$}
            \STATE Control: $u_i(k)\sim\pi(\cdot|y_i(k), z_i(k))$,\\
            \STATE Observation: $y_{i+1}(k)\sim\Phi(\cdot|x_{i+1}(k))$,\\
            \STATE Update: $z_{i+1}(k) \sim \varphi(\cdot|u_{i}(k),y_{i}(k), z_{i}(k))$,\\
            \STATE Receive reward $r_i(k) = r(x_i(k), u_i(k))$.
        \ENDFOR
        \STATE Compute semi-gradient for $\widehat{\cQ}_{k,i}^\pi \defn \langle\beta_k,\psi_i(k)\rangle$: $$g_k = \Big(\sum_{i=0}^{m-1}\gamma^ir_i(k)+\gamma^m\widehat{\cQ}_{k,m}^\pi-\widehat{\cQ}_{k,0}^\pi\Big)\nabla_\beta\widehat{\cQ}_{k,0}^\pi.$$\\
        \STATE Update: ${\beta}_{k+1} = \mathbf{Proj}_{\mathcal{B}_{2}(0,R)}(\beta_k + \alpha g_k).$
    \ENDFOR
  \STATE Return $\widehat{\cQ}_K^\pi(\cdot) = \langle\frac{1}{K}\sum_{k<K}\beta_k,\psi(\cdot)\rangle$.
  \end{algorithmic}
\end{algorithm}

\begin{algorithm}[bht]
    \caption{Finite-State Natural Actor Critic: {\tt FS-NAC}}
    \label{alg:nac}
\begin{algorithmic}[1]
   \STATE {\bfseries Input:} $T$: time-horizon, $N$: number of SGD steps, $\zeta,\eta$: step-sizes, $R$: projection radius.\\
   \STATE {\bfseries Initialization:} $\theta_0 = 0;$\hfill \textbackslash\textbackslash  {\tt ~max-entropy policy}\\
   \FOR{$t=0$ {\bfseries to} $T-1$}
        \STATE Obtain $\widehat{\cQ}_T^{\pi_t}$ by $m$-step TD learning (Alg. \ref{alg:td})\\
        \STATE Initialize: $w_t(0) = 0$\\
        \FOR{$k=0$ {\bfseries to} $N-1$}
            \STATE Obtain $(y_k,z_k)\sim \mathbf{d}_\xi^{\pi_t}$ and $u_k\sim\pi_t(\cdot|y_k,z_k)$\\
            \STATE $\widetilde{w}_{t}(k+1) = w_t(k)-\zeta\cdot \nabla \mathcal{L}_t(w_t(k);u_k,y_k,z_k)$, \label{alg-sgd:line1}
            \STATE $w_t(k+1)=\mathbf{Proj}_{\cB_2(0,R)}\big(\widetilde{w}_t(k+1)\big)$ \label{alg-sgd:line2}
        \ENDFOR
        \STATE $\theta_{t+1} = \theta_t + \eta \frac{1}{N}\sum_{k<N}w_t(k)$.
    \ENDFOR
    \end{algorithmic}

\end{algorithm}

\textbf{Step 2: (Actor)} Let $\mathfrak{H}_t$ be the $\sigma$-field generated by the samples used up to (excluding) iteration $t$, and in the computation of $\widehat{\cQ}_K^{\pi_t}$. Then, for any $t\in\mathbb{N}$, we aim to solve
\begin{equation}
    w^\star_t \in \underset{w\in\mathcal{B}_d(0,R)}{\min} \bE[\mathcal{L}_t(w; U, Y, Z)|\mathfrak{H}_t],
    \label{eqn:cfa}
\end{equation}
where $\mathcal{L}_t(w; u, y, z) = \Big(\langle\nabla\log\pi_t(u|y,z),w\rangle-\widehat{\cA}_K^{\pi_t}(u, y,z)\Big)^2$. In order to solve \eqref{eqn:cfa} by using samples, we initialize $w_t(0) = 0$ and utilize stochastic gradient descent (SGD) as Line \ref{alg-sgd:line1}-\ref{alg-sgd:line2} in Algorithm \ref{alg:nac}. After $N$ iterations, the policy is updated as
    $\theta_{t+1} = \theta_t + \eta\cdot\frac{1}{N}\sum_{k=0}^{N-1}w_t(k).$

\begin{discussion}\normalfont
In the following, we provide an intuitive explanation behind the choices of the methods used for the actor and critic in Algorithm \ref{alg:nac}.
    \smallskip
    \newline \textbf{Why $m$-step TD learning as critic?} The main challenge in estimating $\cQ^\pi$ for a given policy $\pi\in\Pi_{\bZ,\varphi}$ by using temporal difference (TD) learning methods is the so-called \emph{perceptual aliasing} phenomenon \cite{shani2013survey, singh1994learning}, which refers to receiving the same observation $y\in\bY$ for two different (hidden) states $x,x'\in\bX$ with non-zero probability due to the noisy observation channel $\Phi$. As a result of perceptual aliasing, TD(0) does not converge to $\cQ^\pi$ \cite{singh1994learning} since $\mathcal{T}^\pi\cV^\pi\neq \cV^\pi$ for POMDPs, where the Bellman operator $h_0\mapsto\mathcal{T}^\pi(h_0)$ is defined as $$(\mathcal{T}^\pi \cV)(h_0) \defn \bE^\pi[r(X_0,U_0)+\gamma \cV(H_1)|H_0=h_0].$$
    To address perceptual aliasing at the expense of increased sample complexity, we employ multi-step TD learning. The impact of this choice is explicitly characterized in Theorem \ref{thm:m-tdl}.
    \smallskip
    \newline \textbf{Why SGD for policy update?} The exact minimizer $w_t^\star$ of the least-squares problem \eqref{eqn:cfa} yields the natural policy gradient $(1-\gamma)G_t^{\dagger}\nabla V^{\pi_t}(\xi)$, where $$G_t=\bE_{\substack{ (Y,Z)\sim\mathbf{d}_\xi^{\pi_t}\\ U\sim\pi_t(\cdot|Y,Z)}}[\nabla_\theta\log\pi_t(U|Y,Z)\nabla_\theta^\top\log\pi_t(U|Y,Z)],$$ is the Fisher information matrix under $\pi_t$ \cite{kakade2001natural}. In the absence of system model, we utilize projected SGD to approximate $w_t^\star$ by using samples from the system.
\end{discussion}

 \section{Finite-Time Bounds for NAC for POMDPs}\label{sec:analysis}
In this section, we will provide finite-time performance bounds for {\tt FS-NAC}, and identify the impacts of partial observability, function approximation and internal state representation on the global optimality.

Without loss of generality, we assume that $\sup_{(u,y,z)\in\bU\times\bY\times\bZ}\|\psi(u,y,z)\|_2 \leq 1.$ Given a projection radius $R > 0$, the function space defined by $\Psi$ is
\begin{equation*}
    \mathcal{F}_{\Psi}^R = \{(u,y,z)\mapsto\langle\beta, \psi(u,y,z)\rangle:\beta\in\mathcal{B}_{2}(0,R)\}.
\end{equation*}
For any $\pi\in\Pi_{\bZ,\varphi}$, $(\mathbf{d}_\xi^\pi\otimes\pi)(u,y,z) \stackrel{\rm def}{=} \mathbf{d}_\xi^\pi(y,z)\pi(u|y,z),$ denotes the discounted observation-internal state-action visitation distribution under $\pi$.

\subsection{Performance Bounds for the Critic}
In the following, we present finite-time performance bounds for the $m$-step TD learning algorithm for any given FSC $\pi\in\Pi_{\bZ,\varphi}$.

\begin{theorem}[Finite-time bounds for $m$-step TD learning]\normalfont
For any $\pi\in\Pi_{\bZ,\varphi}$ and $m \geq 1$, we have the following bound under Algorithm \ref{alg:td} with $\alpha = \frac{1}{\sqrt{K}}$ and given radius $R>0$:
\begin{align*}
    \sqrt{\bE\Big[\|\cQ^\pi-\widehat{\cQ}^\pi_K\|_{\mathbf{d}_\xi^\pi\otimes\pi}^2\Big]} &\leq \underbrace{\frac{\cO(\frac{R+1}{1-\gamma})}{K^{1/4}\sqrt{1-\gamma^m}}}_{\rm TD~learning~error} + \underbrace{\frac{\epsilon_{\rm app}(R)}{1-\gamma^m}}_{\rm approx.~error}+ \underbrace{\epsilon_{\rm pa}^\pi(\gamma, m, R)}_{\substack{\rm  perceptual~aliasing\\\rm error}}
\end{align*}
where $\widehat{\cQ}_K^\pi(\cdot)=\langle\frac{1}{K}\sum\limits_{k<K}\beta_k,\psi(\cdot)\rangle$ is the output of Alg. \ref{alg:td},  $$\epsilon_{\rm app}(R) \defn \min_{f\in\mathcal{F}_\Psi^R}\|f-\cQ^\pi\|_{\mathbf{d}_\xi^\pi\otimes \pi},$$ is the function approximation error, and the perceptual aliasing error is $$\epsilon_{\rm pa}(\gamma,m,R) = \cO\Big(\gamma^{m/2}{\mathsf{poly}}(R,\frac{1}{1-\gamma})\Big).$$

\label{thm:m-tdl}
\end{theorem}

The detailed form of $\epsilon_{\rm pa}(\gamma,m,R)$ and the proof of Theorem \ref{thm:m-tdl} are in Appendix \ref{app:m-tdl}.

\begin{discussion}\normalfont
From Theorem \ref{thm:m-tdl}, we have the following observations.

    1. The term $\epsilon_{\rm pa}(\gamma, m, R)$ is unique to POMDPs and does not appear in the case of (fully observable) MDPs. Specifically, it quantifies the impact of perceptual aliasing on policy evaluation in the partially observed setting.

    2. For any $\gamma$-discounted-reward POMDP, $\epsilon_{\rm pa}(\gamma, m, R)$ decays at a geometric rate $\cO(\gamma^{m/2})$. As such, the required $m$ to achieve a given target error $\epsilon$ is $\cO\left(\log_{1/\gamma}\left(\frac{1}{(1-\gamma)\epsilon}\right)\right)$.

    3. The hidden terms in $\epsilon_{\rm pa}$, which measure the discrepancy between the belief distributions under perfect observations and partial observations, vanish in the case of MDPs, making $\epsilon_{\rm pa}(\gamma,m,R)=0$ (see Appendix \ref{app:m-tdl}). Particularly, for $I_k=(Y_k,Z_k)$, the hidden terms are the factors of the following, which both vanish in MDPs, and are bounded otherwise: $\sum_{k=0}^\infty\gamma^{km}\|b_0(\cdot,I_{(k+1)m})-b_{(k+1)m}(\cdot)\|_{\rm TV}$ and $\|\delta_m^\pi\otimes\pi - \mathbf{d}_\xi^\pi\otimes\pi\|_{\rm TV},$ where $\delta_m^\pi(y,z)\defn\bE_{H_0\sim\mathbf{d}_\xi^\pi}\Big[\bP\big(I_m=(y,z)|H_0\big)\Big].$ Thus, by setting $m=1$ and considering MDPs, Theorem \ref{thm:m-tdl} reduces to the finite-time bounds for projected TD(0) in the iid setting \cite{bhandari2018finite}.

    4. At each iteration, $m$ samples are used by $m$-step TD learning. Thus, the sample complexity of $m$-step TD learning is $mK$. As such, there is a tradeoff between accuracy (to avoid perceptual aliasing) and sample complexity in policy evaluation. As such, the sample complexity to achieve a target error $\epsilon > 0$ is $\widetilde{\cO}\Big(\frac{1}{(1-\gamma)^2\epsilon^4}\Big)$.
    
\end{discussion}

\subsection{Finite-Time Bounds for {\tt FS-NAC} for POMDPs}
First, we consider the performance of the natural actor-critic for a general finite-state controller, and characterize the function approximation error, statistical error and inference error.

\begin{definition}\label{def:cfa}\normalfont
    For a given set of feature vectors $\Psi$ and projection radius $R > 0$, let
    \begin{equation*}
        \ell_{\rm CFA}(R, \Psi) \defn \sup_{t\geq 0}~\bE\Big[\inf_{f\in\mathcal{F}_\Psi^R}\|f-\cQ^{\pi_t}\|_{\mathbf{d}_\xi^{\pi_t}\otimes\pi_t}\Big|\theta_t\Big],
    \end{equation*}
    be the compatible function approximation error.
\end{definition}

Definition \ref{def:cfa} characterizes the representation power of the function approximation used in policy parameterization and policy evaluation. 

\begin{assumption}[Concentrability coefficient]\label{assumption:cc}
Let
\begin{equation}
    C_t = \bE\Big[\Big|\frac{(\mathbf{d}_\xi^{\pi^*}\otimes\pi^*)(U,Y,Z)}{({\mathbf{d}}_\xi^{\pi_t}\otimes\pi_t)(U,Y,Z)}\Big|^2\Big|\theta_t\Big],
\end{equation}
where the conditional expectation is over $(Y,Z)\sim\mathbf{d}_\xi^{\pi_t}$ and $U\sim\pi_t(\cdot|Y,Z)$. We assume there exists $\bar{C}_\infty < \infty$ such that $\sup_{t\in\mathbb{N}}\bE[C_t]\leq \bar{C}_\infty.$
\end{assumption}

\begin{discussion}[Difficulty of exploration in POMDPs]\normalfont
    Assumption \ref{assumption:cc}, which asserts that the concentrability coefficient is bounded for all iterations throughout the policy optimization, is standard in PG/NPG methods for POMDPs \cite{agarwal2021theory, cai2019neural, mei2020global,thrun1992cient,bhandari2019global}. On the other hand, it is significantly stronger for POMDPs since the probability simplex is over $\bU\times\bY\times\bZ$, which may be significantly larger compared to $\bX\times\bU$, which is the case in MDPs. This also suggests the following dilemma: in order to achieve better performance, one needs a larger internal state space (i.e., memory) $\bZ$, but then the exploration becomes more complex as measured by $\bar{C}_\infty$.
\end{discussion}

\begin{theorem}[Finite-time bounds for {\tt FS-NAC}]\normalfont
Consider the finite-state natural actor-critic with internal state $(\bZ,\varphi)$. Then, Algorithm \ref{alg:nac} with step-sizes $\alpha = \frac{1}{\sqrt{K}}$, $\zeta = \frac{R\sqrt{1-\gamma}}{\sqrt{2N}}$ and  $\eta = \frac{1}{\sqrt{T}}$ achieves the following bound:
\begin{equation*}
    (1-\gamma)\min_{t<T}\bE[\cV^{\pi^*}(\xi)-\cV^{\pi_t}(\xi)] \leq \frac{\log|\bU|+R^2}{\sqrt{T}}
    + 8\bar{C}_\infty\Big({\epsilon_{\rm critic}^m(K, R)} + {\epsilon_{\rm actor}(N, R)}\Big)+ {2\epsilon_{\rm inf}^{\pi^*}(\xi),}
\end{equation*}
where
\begin{equation*}
\epsilon^m_{\rm critic}(K,R) = \frac{\cO(\frac{R+1}{1-\gamma})}{K^{1/4}(1-\gamma^m)^{1/2}}+\frac{\ell_{\rm CFA}(R, \Psi)}{1-\gamma^m}+\frac{1}{T}\sum_{t=0}^{T-1}\bE[\epsilon_{\rm pa}^{\pi_t}(\gamma, m, R)],\end{equation*} 
is the error in the critic, $$\epsilon_{\rm actor}(N,R) = \sqrt{\frac{2-\gamma}{1-\gamma}\cdot\frac{R}{\sqrt{N}}}+\ell_{\rm CFA}(R, \Psi),$$ is the error in the actor updates, and
\begin{equation*}
    \epsilon_{\rm inf}^{\pi^*}(\xi) = \bE^{\pi^*}\Big[\sum_{k=0}^\infty\gamma^k\|b_k(\cdot)-b_0(\cdot,I_k)\|_{\rm TV}\Big|I_0\sim\xi\Big],
\end{equation*}
with $I_k\defn(Y_k,Z_k)$ is the inference error for using the internal state representation $(\bZ,\varphi)$.
\label{thm:nac-fsc}
\end{theorem}

The proof of Theorem \ref{thm:nac-fsc} is presented in Appendix \ref{app:fs-nac}. The general strategy in the proof is to use the Lyapunov function $$\Lambda(\pi) = \sum_{(y,z)\in\bY\times\bZ}\mathbf{d}_\xi^{\pi^*}(y,z)\mathfrak{D}_{\rm KL}\big(\pi^*(\cdot|y,z)\big\|\pi(\cdot|y,z)\big),$$ for finite-state controllers, akin to the case of MDPs \cite{agarwal2021theory}. However, the long-term statistical dependencies due to the use of memory to solve POMDPs constitutes the main challenge. In particular, performance difference lemma for POMDPs (see Lemma \ref{lemma:pdl} in Appendix \ref{app:fs-nac}) is challenging under partial observability.

\begin{discussion}\normalfont
    The bound in Theorem \ref{thm:nac-fsc} can be decomposed into three parts.
\begin{itemize}[leftmargin=*]
\item \textit{Inference error:} The inference error at stage $k$ is 
$\|b_k(\cdot,H_k)-b_0(\cdot,I_k)\|_{TV}$,
where $H_k$ is the complete history up to time $k$ and $I_k=(Y_k,Z_k)$ is the information used by the controller. If the internal state $\{Z_k:k\geq 0\}$ with $(\bZ,\varphi)$ provides a good ``temporal" approximation, i.e., summarizes the history properly, then the inference error is small. The error is due to employing an internal state to compress the history. For the specific case of sliding-window controllers, we will expand this discussion in the following subsection.

    \item \textit{Error in actor}: This corresponds to the combination of statistical error due to using \textbf{Proj-}SGD for policy update, and using a function approximation scheme for policy parameterization. To achieve $\epsilon$-optimality up to a function approximation error, which depends on the expressive power of $\Psi$, one should choose $N = O(\frac{1}{\epsilon^4})$. 
    
    \item \textit{Error in critic}: This corresponds to the error in the critic in every stage of the policy optimization. Note that, by Theorem \ref{thm:m-tdl}, $\frac{1}{T}\sum_t\bE[\epsilon_{\rm PA}^{\pi_t}(\gamma,m,R)] = \exp(-\Omega(m)).$ Therefore, in order to achieve $\epsilon$-optimality, one has to choose $K = \cO(\frac{1}{\epsilon^4})$, and apply $m$-step TD learning with $m = \cO\big(\log_{1/\gamma}(1/\epsilon)\big)$ to control the error due to perceptual aliasing in policy evaluation (see Section \ref{sec:critic}).
    \end{itemize}
\end{discussion}

\subsection{Memory-Inference Error Tradeoff for Sliding-Window Controllers}
The choice of internal state representation determines a tradeoff between memory complexity and the inference error $\epsilon_{\rm inf}^{\pi^*}$. For the special class of sliding-window controllers with block-length $n \geq 0$, we can explicitly characterize this tradeoff under the following conditions.

\begin{condition}[Persistence of excitation under $\pi^*$]\normalfont
There exist $\rho\in(0,1)$ and $\bar{\mu}\in\Sigma(\bU)$ such that $\rho\cdot\bar{\mu}(u) \leq \pi^*(u|y,z) \leq \frac{1}{\rho}\cdot\bar{\mu}(u),$ for all $(u,y,z)\in\bU\times\bY\times\bZ$.
\label{cond:poe}
\end{condition}
Condition \ref{cond:poe} implies that $\mathsf{supp}(\pi^*(\cdot|y,z))$ is the same for all $(y,z)\in\bY\times\bZ$, which holds if $\pi^*$ satisfies the persistence of excitation condition. Note that, unlike MDPs, there may be only strictly non-deterministic policies for POMDPs \cite{singh1994learning}. Furthermore, if one employs entropy regularization within the NAC framework, which is commonly employed in practice, $\pi^*$ automatically satisfies Condition \ref{cond:poe} \cite{shani2020adaptive, cayci2021linear, lee2020stochastic}.
\begin{condition}[Minorization-majorization]\normalfont
There exist $\epsilon_0\in(0,1)$, $m_0\geq 1$ and $\nu\in\Sigma(\bX\times\bY^{m_0}\times\bU^{m_0})$ such that for all $x_{m_0},x_0\in\bX,(y_1^{m_0},u^{{m_0}-1})\in\mathsf{Y}^{m_0}\times \mathsf{U}^{m_0}$,
\begin{equation*}
    \epsilon_0\cdot\nu(x_{m_0},y_1^{m_0},u^{m_0-1}) \leq \widetilde{P}_{m_0}(x_{m_0},y_1^{m_0},u^{m_0-1}|x_0) \leq \frac{1}{\epsilon_0}\cdot\nu(x_{m_0},y_1^{m_0},u^{m_0-1})
\end{equation*}
 where 
$\widetilde{P}_m(x_m,y_1^m,u^{m-1}|x_0)=\sum_{x_1^{m-1}}\prod_{j=0}^{m-1}\bar{\mu}(u_j)\cP(x_{j+1}|x_j,u_j)\Phi(y_{j+1}|x_{j+1}).$
\label{cond:minorization}
\end{condition}
Condition \ref{cond:minorization} is an ergodicity condition, and one of the implications is that every hidden state is visited within a finite time interval. This is akin to the standard ergodicity conditions in \cite{ortner2020regret, wei2021last} for MDPs, but it is \textit{considerably} stronger than them because of the complications due to partial observability. For hidden Markov chains (HMCs), along with a non-degeneracy condition on $\Phi$, Condition \ref{cond:minorization} implies filter stability for any finite $\bX,\bY$ when $\{X_k:k\geq 0\}$ is irreducible and aperiodic \cite{van2008hidden}. For further discussion on Conditions \ref{cond:poe}-\ref{cond:minorization}, see Appendix \ref{app:filter-stability}.

The following result characterizes the tradeoff between the inference error and the memory complexity for the specific case of sliding-window controllers.

\begin{proposition}[Memory-performance tradeoff]
    Under Conditions \ref{cond:poe}-\ref{cond:minorization}, we have 
    \begin{equation}
        \epsilon_{\rm inf}^{\pi^*}(\xi) \leq \frac{1}{(1-\gamma)}\cdot \cO\Big((1-\rho^{2m_0-2}\cdot\epsilon_0^2)^{\lfloor\frac{n}{m_0}\rfloor}\Big),~n\geq 1.
    \end{equation}
    \label{prop:filter-stability}
\end{proposition} Proposition \ref{prop:filter-stability} is an extension of Theorem 5.4 in \cite{van2008hidden} for hidden Markov chains to the case of POMDPs, which accounts for the control (see the following discussion for details). We provide a detailed proof in Appendix \ref{app:filter-stability}.

\begin{discussion}\normalfont We have the following remarks on the memory-inference tradeoff for the specific case of sliding-window controllers, as outlined in Prop. \ref{prop:filter-stability}.
   
    \textbullet~~Proposition \ref{prop:filter-stability} implies that the inference error for {\tt SW-NAC} of window-length $n \geq 1$ decays at a rate $e^{-\Omega(n/m_0)}$ under Condition \ref{cond:minorization}. Hence, a target inference error $\epsilon$ requires a memory complexity of $O(m_0\log(1/\epsilon))$ where $m_0$ is specified in Condition \ref{cond:minorization}. 
    
    \textbullet~~ In order to gain intuition about Proposition \ref{prop:filter-stability}, note that 
\begin{align}
\begin{aligned}
b_0(\cdot,I_k) &= F^{(n)}\Big(b_{-n}(\cdot,Y_{k-n}),Y_{k-n+1}^k,U_{k-n}^{k-1}\Big),\\
b_k(\cdot, H_k) &= F^{(n)}\Big(b_{k-n}(\cdot,H_{k-n}),Y_{k-n+1}^k,U_{k-n}^{k-1}\Big),
\end{aligned}
\label{eqn:inference-error-sb}
\end{align}
where $F^{(n)}$ is the $n$-step Bayes filter (see Section \ref{subsec:belief}), $Z_k=(Y_{k-n}^{k-1},U_{k-n}^{k-1})$ and $I_k=(Y_k,Z_k).$
Thus, the inference error at time $k\geq 0$ is the total-variation distance between the probability measures in \eqref{eqn:inference-error-sb}, which start from two different priors $b_{-n}(\cdot,Y_{k-n})$ and $b_{k-n}(\cdot,H_{k-n})$, and are updated by using the same samples $(Y_{k-n+1}^k,U_{k-n}^{k-1})$ which are obtained under $\pi^*$. Proposition \ref{prop:filter-stability} implies that, if the underlying Markov chain is ergodic in the sense of Condition \ref{cond:minorization}, different priors are forgotten at a geometric rate in $n$, similar to HMCs \cite{cappe2005inference, van2008hidden}.

\textbullet~~Under an $n$-step decodability assumption akin to \cite{efroni2022provable,wang2022embed, uehara2022provably}, we observe that $b_0(\cdot,I_k) = b_k(\cdot,H_k)$ in \eqref{eqn:inference-error-sb} since different priors do not affect $n$-step belief in that case. This implies that $\epsilon_{\rm inf}^{\pi^*}(\xi) = 0$ by using Prop. \ref{prop:filter-stability}. Thus, $n$-step decodability is a realizability assumption, and our results cover the unrealizable case by characterizing the inference error $\epsilon_{\rm inf}^{\pi^*}(\xi)$.

\textbullet~~In the case of finite-state POMDPs, for tabular Q-learning, a different characterization of the inference error was presented in \cite{kara2021convergence} under an assumption on the Dobrushin coefficient. In this work, we prove bounds on the inference error in arbitrarily large state-observation spaces within the natural actor-critic framework under different conditions. One major difference is in terms of RL approach: we adopt a direct policy optimization instead of a value-based method. A part of these results are inspired by the connection between the inference error and the notion of filter stability which was observed in \cite{kara2021convergence}. However, one key difference is that, in our analysis, it is not possible to directly use existing filter stability results for HMCs (e.g., \cite{van2008hidden, cappe2005inference}) in POMDPs, since, while in the case of HMCs, the current observation is only a function of the current hidden state, the current control action is potentially a function of all past observations and control actions. Therefore, in Appendix \ref{app:filter-stability}, we extend the filter stability results in \cite{van2008hidden} for HMCs to the case of POMDPs.
\end{discussion}

\section{Conclusion}
In this paper, we proposed a natural actor-critic method for POMDPs, which employs an internal state for memory compression, and a multi-step TD learning algorithm for the critic. We established bounds on the sample complexity and memory complexity of the finite-state NAC method. Our analysis shows that under ergodicity and concentrability coefficient conditions, sliding-window NAC with sufficiently large window-length can achieve global optimality up to the function approximation error. On the other hand, compared to their MDP counterparts, these conditions are considerably stronger, which underlines the difficulty of solving POMDPs.

\appendix
\section{Analysis of $m$-Step TD Learning for POMDPs}\label{app:m-tdl}
First, we provide the complete statement of Theorem \ref{thm:m-tdl}, which contains the omitted terms.

\begin{duplicate}{Theorem~\ref{thm:m-tdl}}\normalfont ~For any $\pi\in\Pi_{\bZ,\varphi}$ and $m \geq 1$, we have the following bound under Algorithm \ref{alg:td} with $\alpha = \frac{1}{\sqrt{K}}$ and given radius $R>0$:
\begin{align*}
    \sqrt{\bE\Big[\|\cQ^\pi-\widehat{\cQ}^\pi_K\|_{\mathbf{d}_\xi^\pi\otimes\pi}^2\Big]} &\leq \underbrace{\sqrt{\frac{4R^2+\big(\frac{1}{1-\gamma} + 2R\big)^2}{K^{1/2}(1-\gamma^m)}}}_{\rm TD~learning~error} + \underbrace{\frac{\epsilon_{\rm app}(R)}{1-\gamma^m},}_{\rm approx.~error}+ \underbrace{\epsilon_{\rm pa}^\pi(\gamma, m, R)}_{\substack{\rm  perceptual~aliasing\\\rm error}}
\end{align*}
where $\widehat{\cQ}_K^\pi(\cdot)=\langle\frac{1}{K}\sum\limits_{k<K}\beta_k,\psi(\cdot)\rangle$,  $\epsilon_{\rm app}(R) \defn \min_{f\in\mathcal{F}_\Psi^R}\|f-\cQ^\pi\|_{\mathbf{d}_\xi^\pi\otimes \pi},$ and
\begin{multline}\epsilon_{\rm pa}(\gamma,m,R) = \Big(\frac{R+1}{1-\gamma}\Big)\sqrt{\frac{2\gamma^m\|{\mathbf{\delta}}_{\xi}^\pi\otimes\pi-\mathbf{d}_\xi^\pi\otimes\pi\|_{TV}}{1-\gamma^m}}\\+\cO\Big(\frac{\gamma^m}{1-\gamma}\Big)\Big\|\bE\Big[\sum_{k=0}^\infty\gamma^{km}\|\textbf{b}_{0,(k+1)m}-\textbf{b}_{(k+1)m}\|_{\rm TV}\Big|I_0=\cdot\Big]\Big\|_{\textbf{d}_\xi^\pi\otimes\pi},
\end{multline}
for $
    \textbf{b}_{0,k}\defn\big[b_0(x, I_{k})\big]_{x\in\bX}\in\Sigma(\bX),~~~\mbox{and}~~~\textbf{b}_{k}\defn\big[b_{k}(x,H_{k})\big]_{x\in\bX}\in\Sigma(\bX).$

\end{duplicate}

\begin{proof}[Proof of Theorem \ref{thm:m-tdl}]
Under an FSC $\pi\in\Pi_{\bZ,\varphi}$, for any $(y_0,z_0,u_0)\in\bY\times\bZ\times\bU$, let $\cQ_*^\pi$ be the fixed point of the following equation: 
    \begin{equation}
        \cQ(u_0,y_0,z_0) = \bE^\pi\Big[\sum_{i=0}^{m-1}\gamma^i r(X_i,U_i) + \gamma^m \cQ(U_m,Y_m,Z_m)\Big|Y_0=y_0,Z_0=z_0,U_0=u_0\Big].
    \end{equation} 
    Let $\beta_\pi = \arg\min_{\beta\in\mathcal{B}_2(0,R)}\|\cQ^\pi-\langle \beta,\psi(\cdot)\rangle\|_{\mathbf{d}_\xi^\pi\otimes \pi}$
be the optimal parameter to approximate $\cQ^\pi$ by using features $\Psi$. Also, we define:
\begin{equation*}
    \delta_{m,\xi}^\pi(y,z) \defn \sum_{y_0,z_0}\bP^\pi\Big(Y_m=y,Z_m = z\Big|Y_0=y_0,Z_0=z_0\Big)\mathbf{d}_{\xi}^\pi(y_0,z_0).
\end{equation*} For any $k\geq 0$ and $i\in[m]$, let $\psi_i(k) := \psi(U_i(k),Y_i(k),Z_i(k))\in\bR^d,$ and
    $$g_k = \Big(\sum_{i=0}^{m-1}\gamma^ir(X_i(k),U_i(k)) + \gamma^m\widehat{\cQ}^\pi_{k,m}-\widehat{\cQ}^\pi_{k,0}\Big)\psi_0(k),$$
be the semi-gradient, where $\widehat{\cQ}_{k,i}^\pi = \langle\beta_k,\psi_i(k)\rangle,$ is the value function estimate at time $t$. Let $\widetilde{\beta}_{k+1} = \beta_k + \alpha\cdot g_k,$ which implies $\beta_k = \mathbf{Proj}_{\mathcal{B}_{2}(0,R)}\{\widetilde{\beta}_k\},$ where $\mathbf{Proj}_\mathcal{C}\{s\} = \arg\min_{s'\in\mathcal{C}}\|s-s'\|_2$ is the projection operator onto $\mathcal{C}\subset\bR^{d}$. Similar to the analysis of TD(0) with function approximation for MDPs \cite{tsitsiklis1996analysis, bhandari2018finite, cai2019neural}, we consider the following Lyapunov function: $\mathcal{L}(\beta) = \|\beta-\beta_\pi\|_2^2,$
where $\beta_\pi = \underset{\beta \in \mathcal{B}_{2}(0,R)}{\arg\min}\|\cQ^\pi(\cdot)-\langle\beta,\psi(\cdot)\rangle\|_{\mathbf{d}_\xi^\pi\otimes\pi}^2$ is the optimal approximator. Since $\mathcal{B}_{2}(0,R)$ is a convex subset of $\bR^{d}$ and $\mathbf{Proj}_\mathcal{C}$ is non-expansive for convex $\mathcal{C}$, we have
\begin{equation*}
    \cL(\beta_{k+1}) \leq \|\tilde{\beta}_{k+1}-\beta_\pi\|_2^2 \leq \cL(\beta_k) + 2\alpha \langle g_k, \beta_k-\beta_\pi\rangle + \alpha^2\|g_k\|_2^2,~\forall k \geq 0.
\end{equation*}
Since $\sup_{(u,y,z)\in\bU\times\bY\times\bZ}\|\psi(u,y,z)\|_2 \leq 1$ and $\|\beta_k\|_2\leq R$ due to projection, we have $$\sup_{k\geq 0}\|g_k\|_2 \leq \frac{1-\gamma^m}{1-\gamma} + (1+\gamma^m)R = 2R+\frac{1}{1-\gamma}=:G_{\rm max},~\mbox{w.p.}~1.$$ Let $\mathfrak{G}_k = \sigma\big(Y_i(\tau),Z_i(\tau),U_i(\tau), i\leq m, \tau \leq k\big), \mathfrak{F}_k = \sigma\big(U_0(k),Y_0(k),Z_0(k)\big), \bE_k[\cdot] = \bE[\cdot|\mathfrak{G}_{k-1}].$
Then, the Lyapunov drift is as follows:
\begin{equation}
    \bE_k\big[\cL(\beta_{k+1})-\cL(\beta_k)\big|\mathfrak{F}_k\big] \leq 2\alpha\bE_k\big[\langle g_k, \beta_k-\beta_\pi\rangle\big|\mathfrak{F}_k\big] + \alpha^2G_{\rm max}^2.
    \label{eqn:lyapunov-drift}
\end{equation}
Now, we focus on the term $\bE_k[\langle g_k, \beta_k-\beta_\pi\rangle|\mathfrak{F}_k]$. Let $\widetilde{\cQ}_*^\pi(\cdot) \defn \langle \beta_\pi, \psi(\cdot)\rangle.$ We have
\begin{equation}
    \bE_k[\langle g_k, \beta_k-\beta_\pi\rangle|\mathfrak{F}_k] = \Big(\bE_k\Big[\sum_{i=0}^{m-1}\gamma^kr(X_i(k),U_i(k)) + \gamma^m\widehat{\cQ}^\pi_{k,m}\Big|\mathfrak{F}_k\Big]-\widehat{Q}^\pi_{k,0}\Big)\langle\beta_k-\beta_\pi,\psi_0(k)\rangle.
    \label{eqn:cross-term-1}
\end{equation}

Under Assumption \ref{assumption:sampling-oracle}, for the fixed point $\cQ_*^\pi$, we have
\begin{equation}
    \bE_k[\sum_{i=0}^{m-1}\gamma^ir(X_i(k),U_i(k))|\mathfrak{F}_k] = \cQ_*^\pi(U_0(k),Y_0(k),Z_0(k))-\gamma^m \bE_k[\cQ_*^\pi(U_m(k),Y_m(k),Z_m(k))|\mathfrak{F}_k].
    \label{eqn:m-step-sum}
\end{equation}
For notational convenience, we use $(U_0,Y_0, Z_0)$ instead of $(U_0(k),Y_0(k),Z_0(k))$, below. Substituting \eqref{eqn:m-step-sum} into \eqref{eqn:cross-term-1}, and expanding the multiplicative terms, we obtain the following:
\begin{align*}
    \bE_k&[\langle g_k, \beta_k-\beta_\pi\rangle|\mathfrak{F}_k] = \underbrace{-\Big(\cQ_*^\pi(U_0,Y_0,Z_0)-\widehat{\cQ}_k^\pi(U_0,Y_0,Z_0)\Big)^2}_{(i)}\\&
    + \underbrace{\Big(\cQ_*^\pi(U_0,Y_0,Z_0)-\widehat{\cQ}^\pi_k(U_0,Y_0,Z_0)\Big)\cdot\Big({\cQ}_*^\pi(U_0,Y_0,Z_0)-\widetilde{\cQ}^\pi_*(U_0,Y_0,Z_0)\Big)}_{(ii)}\\&\underbrace{
    -\gamma^m\bE_k[\widehat{\cQ}_k^\pi(U_m,Y_m,Z_m)-Q_*^\pi(U_m,Y_m,Z_m)|\mathfrak{F}_k]\cdot\Big(\cQ_*^\pi(U_0,Y_0,Z_0)-\widehat{\cQ}_k^\pi(U_0,Y_0,Z_0)\Big)}_{(iii)}\\&\underbrace{
    +\gamma^m\bE_k[\widehat{\cQ}_k^\pi(U_m,Y_m,Z_m)-\cQ_*^\pi(U_m,Y_m,Z_m)|\mathfrak{F}_k]\cdot({\cQ}_*^\pi(U_0,Y_0,Z_0)-\widetilde{\cQ}^\pi_*(U_0,Y_0,Z_0)\Big)}_{(iv)}.
\end{align*}
First, we take expectation of the terms $(i)$-$(iv)$ above over $\mathfrak{F}_k=\sigma(U_0(k),Y_0(k),Z_0(k))$ given $\mathfrak{G}_{k-1}$ by using the tower property of expectation.

\begin{itemize}[leftmargin=*]
\item Expectation of $(i)$ is equal to $-\|\cQ_*^\pi-\widehat{\cQ}_k^\pi\|_{\mathbf{d}_\xi^\pi\otimes\pi}^2$. 
\item For $(ii)$, by Cauchy-Schwarz inequality:
\begin{multline*}
    \bE_k\Big[\Big(\cQ_*^\pi(U_0,Y_0,Z_0)-\widehat{\cQ}^\pi_k(U_0,Y_0,Z_0)\Big)\cdot\Big({\cQ}_*^\pi(U_0,Y_0,Z_0)-\widetilde{\cQ}^\pi_*(U_0,Y_0,Z_0)\Big)\Big]\\\leq \|\cQ_*^\pi-\widehat{\cQ}_k^\pi\|_{\mathbf{d}_\xi^\pi\otimes\pi}\cdot \|\cQ_*^\pi-\widetilde{\cQ}_*^\pi\|_{\mathbf{d}_\xi^\pi\otimes\pi}.
\end{multline*}
\item For $(iii)$ and $(iv)$, note that:
\begin{multline*}
    \bE_k[\big(\widehat{\cQ}_k^\pi(U_m,Y_m,Z_m)-\cQ_*^\pi(U_m,Y_m,Z_m)\big)^2|\mathfrak{F}_k] \\= \sum_{u_m,y_m,z_m}(\mathbf{d}_\xi^\pi\otimes\pi)(u_m,y_m,z_m)\big(\widehat{\cQ}_k^\pi(u_m,y_m,z_m)-\cQ_*^\pi(u_m,y_m,z_m)\big)^2,
\end{multline*}
Note that we have with probability 1 $\sup_{u,y,z}\cQ^\pi(u,y,z) \leq \frac{1}{1-\gamma}$, $\sup_{u,y,z}\widetilde{\cQ}^\pi(u,y,z) \leq R$, and $\sup_{k\geq 0}\max\{|\widehat{\cQ}_{k,0}^\pi|,|\widehat{\cQ}_{k,m}^\pi|\} \leq R$ since $\beta_\pi\in\mathcal{B}_{2}(0,R)$ and $\beta_k\in\mathcal{B}_{2}(0,R)$ for all $k\geq 0$. Therefore, by using the uniform error based on the aforementioned bounds, the tower property of expectation, and the inequality $\sqrt{a+b}\leq \sqrt{a}+\sqrt{b}$ for any $a,b\in\bR_+$,
\begin{multline*}
    \sqrt{\bE_k[\big(\widehat{\cQ}_k^\pi(U_m,Y_m,Z_m)-\cQ_*^\pi(U_m,Y_m,Z_m)\big)^2]} \\ \leq \big\|\widehat{\cQ}_k^\pi-\cQ_*^\pi\big\|_{\mathbf{d}_\xi^\pi\otimes\pi} + \Big(R+\frac{1}{1-\gamma}\Big)\sqrt{\|\delta_{m,\xi}^\pi\otimes\pi-\mathbf{d}_\xi^\pi\otimes\pi\|_{\rm TV}},
    \label{eqn:cs-inequality}
\end{multline*}
Hence, using \eqref{eqn:cs-inequality} in $(iii)$ and $(iv)$, then taking expectation over $\mathfrak{G}_{k-1}$, we obtain:
\begin{multline}
    \bE[\langle g_k, \beta_k-\beta_\pi\rangle] \leq (1+\gamma^m)\bE[\|\widehat{\cQ}_k^\pi-\cQ_*^\pi\|_{\mathbf{d}_\xi^\pi\otimes\pi}]\|\widetilde{\cQ}_*^\pi-\cQ_*^\pi\|_{\mathbf{d}_\xi^\pi\otimes\pi}
    \\ -(1-\gamma^m)\bE\|\widehat{\cQ}_k^\pi-\cQ_*^\pi\|_{\mathbf{d}_\xi^\pi\otimes\pi}^2
    + 2\gamma^m\Big(\frac{R+1}{1-\gamma}\Big)^2\sqrt{\|\delta_{m,\xi}^\pi\otimes\pi- \mathbf{d}_\xi^\pi\otimes\pi\|_{\rm TV}}.
    \label{eqn:expected-drift-a}
\end{multline}
\end{itemize}

Let $\Delta_{k,\pi}^2 = \bE\|{\cQ}_*^\pi-\widehat{\cQ}_k^\pi\|_{\mathbf{d}_\xi^\pi\otimes\pi}^2$, and recall that $\epsilon_{\rm app}(R) = \min_{\beta\in\mathcal{B}_{2}(0,R)}\|\cQ^\pi-\langle\beta,\psi(\cdot)\|_{\textbf{d}_\xi^\pi\otimes\pi}$. For notational convenience, let $\ell^*_\pi = \frac{1+\gamma^m}{2(1-\gamma^m)}\cdot\epsilon_{\rm app}(R).$ Then, \eqref{eqn:expected-drift-a} can be written as follows:
\begin{align}
    \bE\langle g_k, \beta_k-\beta_\pi\rangle \leq -(1-\gamma^m)\Big(\Delta_{k,\pi}^2 -2 \ell^*_\pi\cdot\Delta_{k,\pi}\Big) + 2\gamma^m\Big(\frac{R+1}{1-\gamma}\Big)^2\sqrt{\|\delta_{m,\xi}^\pi\otimes\pi-\mathbf{d}_\xi^\pi\otimes\pi\|_{\rm TV}}\label{eqn:expected-drift}.
\end{align}
Taking expectation of \eqref{eqn:lyapunov-drift}, and using the bound \eqref{eqn:expected-drift}, we have the expected drift bound
\begin{multline}
    \bE[\cL(\beta_{k+1})-\cL(\beta_k)] \leq -2\alpha(1-\gamma^m)(\Delta_{k,\pi}-\ell_\pi^*)^2\\+2\alpha(1-\gamma^m)(\ell_\pi^*)^2 + 4\alpha\gamma^m\Big(\frac{R+1}{1-\gamma}\Big)^2\sqrt{\|\delta_{m,\xi}^\pi\otimes\pi-\mathbf{d}_\xi^\pi\otimes\pi\|_{\rm TV}}\Big)+\alpha^2(1+2R)^2.
\end{multline}
Telescoping sum over $k=0,1,\ldots,K-1$ yields the following:
\begin{multline}
    \bE[\cL(\beta_K)-\cL(\beta_0)] \leq -2\alpha(1-\gamma^m)\sum_{k<K}\big(\Delta_{k,\pi}-\ell_\pi^*\big)^2\\+2\alpha K\Big((1-\gamma^m)(\ell_\pi^*)^2+4\alpha K\gamma^m\Big(\frac{R+1}{1-\gamma}\Big)^2\sqrt{\|\delta_{m,\xi}^\pi\otimes\pi-\mathbf{d}_\xi^\pi\otimes\pi\|_{\rm TV}}\Big) + \alpha^2K(1+2R)^2,
\end{multline}
Note that $\cL(\beta_K) \geq 0$ and $\cL(\beta_0) = \|\beta_0-\beta_\pi\|_2^2\leq R^2$, which implies that:
\begin{multline*}
    \frac{1}{K}\sum_{k=0}^{K-1}\Big(\Delta_{k,\pi}-\ell_\pi^*\Big)^2 \leq \frac{1}{2\alpha K}\|\beta_0-\beta_\pi\|_2^2 + (\ell_\pi^*)^2 + 2\gamma^m\Big(\frac{R+1}{1-\gamma}\Big)^2\sqrt{\|\delta_{m,\xi}^\pi\otimes\pi-\mathbf{d}_\xi^\pi\otimes\pi\|_{\rm TV}}\\+\alpha (1+2R)^2.
\end{multline*}
First, by Jensen's inequality,
\begin{equation}
    \frac{1}{K}\sum_{k=0}^{K-1}\Delta_{k,\pi} \leq \frac{\|\beta_0-\beta_\pi\|_2}{\sqrt{2\alpha K}} + 2\ell_\pi^* + \sqrt{2\gamma^m\Big(\frac{R+1}{1-\gamma}\Big)^2\sqrt{\|\delta_{m,\xi}^\pi\otimes\pi-\mathbf{d}_\xi^\pi\otimes\pi\|_{\rm TV}}} + \sqrt{\alpha}(1+2R).
\end{equation}
We have $\sqrt{\bE\|\cQ_*^\pi-\widehat{\cQ}_K^\pi\|_{\mathbf{d}_\xi^\pi\otimes\pi}^2} \leq \frac{1}{K}\sum_{k=0}^{K-1}\Delta_{k,\pi},$ where the equality holds from the linearity of $\widehat{\cQ}_k^\pi$ in $\beta_k$, and the inequality follows from Jensen's inequality. Hence, we have:
\begin{multline}
    \sqrt{\bE\|\cQ_*^\pi-\widehat{\cQ}_K^\pi\|_{\mathbf{d}_\xi^\pi\otimes\pi}^2} \leq 2\ell_\pi^* + \sqrt{2\gamma^m\Big(\frac{R+1}{1-\gamma}\Big)^2\sqrt{\|\delta_{m,\xi}^\pi\otimes\pi- \mathbf{d}_\xi^\pi\otimes\pi\|_{\rm TV}}}\\+\frac{\|\beta_0-\beta_\pi\|_2}{\sqrt{2\alpha K}} + \sqrt{\alpha}(1+2R).
    \label{eqn:bound-fixed-point}
\end{multline}

In order to bound $\sqrt{\bE\|\cQ^\pi-\cQ_*^\pi\|_{\mathbf{d}_\xi^\pi\otimes\pi}^2}$, we use the following lemma, which extends the analysis in \cite{kara2021convergence} to multi-step TD learning, to characterize the fixed point $Q_*^\pi$.
\begin{lemma}
    Let $\pi\in\Pi_{\bZ,\varphi}$ be an FSC. Let
\begin{equation}
    \bar{r}_m(x,y,z,u) = \bE^\pi\Big[\sum_{k=0}^{m-1}\gamma^kr(x_{k},u_{k})\Big|x_{0}=x,y_{0}=y,z_{0}=z,u_{0}=u\Big].
\end{equation}
For any $m \geq 1$,
    \begin{align*}
        \cQ_*^\pi(U_0,Y_0,Z_0) &= \bE^\pi\Big[\sum_{k=0}^\infty\gamma^{km}\sum_{x_{km}\in\bX}b_0(x_{km},I_{km})\bar{r}_m(X_{km},U_{km},Y_{km},Z_{km})\Big|U_0,Y_0,Z_0\Big],\\
        \cQ^\pi(U_0,Y_0,Z_0) &= \bE^\pi\Big[\sum_{k=0}^\infty\gamma^{km}\sum_{x_{km}\in\bX}b_{km}(x_{km})\bar{r}_m(x_{km},U_{km},Y_{km},Z_{km})\Big|U_0,Y_0,Z_0\Big],
    \end{align*}
    where $I_k = (Y_k,Z_k)$. Consequently,
    \begin{multline*}
        \Big|\cQ_*^\pi(U_0,Y_0,Z_0)-{\cQ}^\pi(U_0,Y_0,Z_0)\Big| \\ \leq \frac{2(1-\gamma^m)\gamma^m}{1-\gamma}\bE^\pi\Big[\sum_{k=0}^\infty \gamma^{km}\|b_0(\cdot,I_{(k+1)m})-b_{(k+1)m}(\cdot)\|_{\rm TV}\Big|U_0,Y_0,Z_0\Big].
    \end{multline*}
    \label{lemma:fixed-point}
\end{lemma}

We use Lemma \ref{lemma:fixed-point} to bound $\sqrt{\bE\|Q^\pi-Q_*^\pi\|_{\mathbf{d}_\xi^\pi\otimes\pi}^2}$. Using this bound and \eqref{eqn:bound-fixed-point} with triangle inequality, and substituting the step-size $\alpha = 1/\sqrt{K}$, we conclude the proof.
\end{proof}

\begin{proof}[Proof of Lemma \ref{lemma:fixed-point}]
    Since $\{(X_k,U_k,Y_k,Z_k):k\geq 0\}$ is a Markov chain under an FSC,
\begin{equation*}
    \bE^\pi\Big[\sum_{i=0}^{m-1}\gamma^kr(X_{i+km},U_{i+km})\Big|X_{km}=x,U_{km}=u,Y_{km}=y,Z_{km}=z, H_{km-1}\Big] = \bar{r}_m(x,y,z,u),
\end{equation*}
for any $m\geq 1, t \geq 0$, $(x,y,z,u)\in\bX\times\bY\times\bZ\times\bU$. By using the tower property of conditional expectation, the identities for $Q_*^\pi$ and $Q^\pi$ follow. For the second part of the proof, we simply use the triangle inequality in conjunction with the fact that $H_0 = I_0$, and the upper bound
    $\sup\{|\bar{r}_m(x,u,y,z)|:(x,u,y,z)\in\bX\times\bU\times\bY\times\bZ\}\leq \frac{1-\gamma^m}{1-\gamma},$
concluding the proof.
\end{proof}

\section{Sampling $H_0\sim\xi$ for Sliding-Window Controllers}\label{app:sampling}
The initial distribution for the hidden state is $\vartheta\in\Sigma(\bX)$, which induces the distribution $\xi\in\Sigma(\bY\times\bZ)$ as follows. For a given window-length $n$, the system starts at time $-n$ with $(X_{-n},Y_{-n})\sim \vartheta\otimes\Phi$ where $(\vartheta\otimes\Phi)(x,y)=\vartheta(x)\Phi(y|x)$ for any $(x,y)\in\bX\times\bY$, and obtains $h_0$ by following a given exploratory policy $\tilde{\pi}$
    $x_{k+1} \sim \mathcal{P}(\cdot|x_k, u_k),
    y_{k+1} \sim \Phi(\cdot|x_{k+1}),
    u_{k+1} \sim \tilde{\pi}(\cdot|y^{k+1}_{-n},u^{k}_{-n})$
for $k \in[-n,0)$ with $u_{-n}\sim\tilde{\pi}(\cdot|y_{-n})$. By using this trajectory, the controller obtains $h_0 = (y_0, z_0) = (y_{-n}^0,u_{-n}^{-1})$, which yields the prior $b_0 = F^{(n)}\left(b_{-n}(\cdot,y_{-n}),y_{-n+1}^0, u_{-n}^{-1}\right)$ where $
    b_{-n}(\cdot|y) = \frac{\vartheta(\cdot)\Phi(y|\cdot)}{\sum_{x\in\bX}\vartheta(x)\Phi(y|x)}.$ The initial history $H_0=(Y_0, Z_0)$ is random with the distribution $\xi$, which can be explicitly specified by using $\vartheta, \tilde{\pi}$, $\mathcal{P}$ and $\Phi$ as follows:
\begin{equation*}
    \xi(y_0,z_0) = \sum_{x_{-n}^0\in\bX^{n+1}}\vartheta(x_{-n})\Phi(y_{-n}|x_{-n})\mathsf{p}(x_{-n+1}^0,y_{-n+1}^0,u_{-n}^{-1};x_{-n}, y_{-n}),
\end{equation*}
where \begin{equation*}\mathsf{p}(x_{-n+1}^0,y_{-n+1}^0,u_{-n}^{-1};x_{-n}, y_{-n}) = \prod_{j\in\{-n,\ldots,-1\}}\tilde{\pi}(u_j|y_{-n}^{j},u_{-n}^{j-1})\mathcal{P}(x_{j+1}|x_j,u_j)\Phi(y_{j+1}|x_{j+1}).\end{equation*}

\begin{remark}[Sampling from $\mathbf{d}_\xi^\pi$]
We can obtain samples from the discounted observation-internal state visitation distribution $\mathbf{d}_\xi^\pi$ by using an initial sample from $\xi$ obtained via the above scheme in conjunction with the sequential sampler for state visitation distributions (see Algorithm 1 in \cite{agarwal2021theory} and \cite{konda2003onactor}).
\label{remark:sampling}
\end{remark}

\section{Convergence of {\tt FS-NAC}: Proof of Theorem \ref{thm:nac-fsc}}\label{app:fs-nac}

\subsection{Performance Difference Lemma for POMDPs}
We start with an important lemma for the proof of Theorem \ref{thm:nac-fsc}.
\begin{lemma}
    For any given internal state representation $(\bZ,\varphi)$, initial distribution $\xi\in\Sigma(\bY\times\bZ)$, and pair of finite-state policies $\pi,\pi'\in\Pi_{\bZ,\varphi}$, we have the following bound:
    \begin{equation}
        \cV^{\pi'}(\xi) - \cV^\pi(\xi) \geq \frac{1}{1-\gamma}\bE_{(U,Y,Z)\sim \mathbf{d}_\xi^{\pi'}\otimes\pi'}[\cA^\pi(U,Y,Z)]-\frac{2}{1-\gamma}\epsilon_{\rm inf}^{\pi'}(y_0,z_0),
    \end{equation}
    where $\cA^{\pi}(u,y,z) = \cQ^\pi(u,y,z)-\cV^\pi(y,z)$ is the advantage function under $\pi\in\Pi_{Z,\varphi}$, and
    \begin{equation*}
    \epsilon_{\rm inf}^{\pi'}(y_0,z_0) = \bE^{\pi'}\Big[\sum_{k=0}^\infty\gamma^k\|b_k(\cdot, H_k)-b_0(\cdot,I_k)\|_{\rm TV}\Big|Y_0=y_0,Z_0=z_0\Big].
    \end{equation*}
    \label{lemma:pdl}
\end{lemma}
For the case of (fully observable) MDPs, Lemma \ref{lemma:pdl} reduces to the well-known performance difference lemma proposed in \cite{kakade2002approximately}. In Lemma \ref{lemma:pdl}, inspired by the analyses in \cite{kakade2002approximately} and \cite{kara2021convergence}, we establish the performance difference results for POMDPs, which characterize the impact of partial observability for finite-state controllers.

\begin{proof}
We have the following identity from the definition:
\begin{align*}
    (\cV^{\pi'}-\cV^{\pi})(Y_0,Z_0) &= \bE^{\pi'}\Big[\sum_{k=0}^\infty\gamma^kr(X_k,U_k)\Big|Y_0,Z_0\Big] - \cV^\pi(Y_0,Z_0),\\
    &= \bE^{\pi'}\Big[\sum_{k=0}^\infty\gamma^k\Big(r(X_k,U_k)+\cV^\pi(Y_k,Z_k)-\cV^\pi(Y_k,Z_k)\Big)\Big|Y_0,Z_0\Big] - \cV^\pi(I_0),\\
    &\stackrel{(a)}{=} \bE^{\pi'}\Big[\sum_{k=0}^\infty\gamma^k\Big(r(X_k,U_k)+\gamma \cV^\pi(Y_{k+1},Z_{k+1})-\cV^\pi(Y_k,Z_k)\Big)\Big|Y_0,Z_0\Big],
\end{align*}
where $(a)$ holds since $$\bE\Big[\sum_{k=0}^\infty\gamma^k\cV^\pi(Y_k,Z_k)\Big|Y_0,Z_0\Big] = \cV^\pi(Y_0,Z_0)+\gamma\cdot\bE\Big[\sum_{k=0}^\infty \gamma^k \cV^\pi(Y_{k+1},Z_{k+1})\Big|Y_0,Z_0\Big].$$
Since the Bayes-filtered value function $\cV^\pi$ is not the fixed point of a Bellman equation due to POMDP dynamics, we decompose $(a)$ into two parts as follows:
\begin{align}
\begin{aligned}
    (\cV^{\pi'}-\cV^{\pi})(Y_0,Z_0)=\underbrace{ \bE^{\pi'}\Big[\sum_{k=0}^\infty\gamma^k\Big(r(X_k,U_k)+\gamma \cV_0^\pi(X_{k+1},Y_{k+1},Z_{k+1})-\cV^\pi(Y_k,Z_k)\Big)\Big|Y_0,Z_0\Big]}_{(i)},\\
    + \underbrace{\gamma\cdot\bE^{\pi'}\Big[\sum_{k=0}^\infty\gamma^k\Big(\cV^\pi(Y_{k+1},Z_{k+1})-\cV_0^\pi(X_{k+1},Y_{k+1},Z_{k+1})\Big|Y_0,Z_0\Big]}_{(ii)},
    \label{eqn:pdl-sum-i}
    \end{aligned}
\end{align}
where 
    $\cV_0^\pi(X_0,Y_0,Z_0) = \bE\Big[\sum_{k=0}^\infty\gamma^kr(X_k,U_k)\Big|X_0,Y_0,Z_0\Big]$
is the unfiltered value function.

In what follows, we will bound $(i)$ and $(ii)$ in the above identity.

\textbf{Bounding $(i)$ in \eqref{eqn:pdl-sum-i}:} Since $0\leq \inf\limits_{x\in\bX,u\in\bU}r(x,u) \leq \sup\limits_{x\in\bX,u\in\bU}r(x,u)\leq 1,$ we have the following inequality for any $K>0$ almost surely:
\begin{equation}
    \Big|\sum_{k=0}^K\gamma^k\Big(r(X_k,U_k)+\gamma \cV_0^\pi(X_{k+1},Y_{k+1},Z_{k+1})-\cV^\pi(Y_k,Z_k)\Big)\Big| \leq \frac{2}{(1-\gamma)^2} < \infty.
\end{equation}
Thus, by Lebesgue's dominated convergence \cite{royden1988real}, we can expand $(i)$ in \eqref{eqn:pdl-sum-i} as follows:
\begin{multline*}
    \bE^{\pi'}\Big[\sum_{k=0}^\infty\gamma^k\Big(r(X_k,U_k)+\gamma \cV_0^\pi(X_{k+1},Y_{k+1},Z_{k+1})-\cV^\pi(Y_k,Z_k)\Big)\Big|Y_0,Z_0\Big] \\= \sum_{k=0}^\infty\gamma^k\bE^{\pi'}\Big[r(X_k,U_k)+\gamma \cV_0^\pi(X_{k+1},Y_{k+1},Z_{k+1})-\cV^\pi(Y_k,Z_k)\Big|Y_0,Z_0\Big].
\end{multline*}
For any $k \geq 0$, by the law of iterated expectation, we can write the following:
\begin{multline}
    \bE^{\pi'}[r(X_k,U_k)+\gamma \cV_0^\pi(X_{k+1},Y_{k+1},Z_{k+1})-\cV^\pi(Y_k,Z_k)|Y_0,Z_0] \\= \bE\Big[\bE^{\pi'}\big[r(X_k,U_k)+\gamma \cV_0^\pi(X_{k+1},Y_{k+1},Z_{k+1})\big|H_k,Z_k\big]-V^\pi(Y_k,Z_k)\Big|Y_0,Z_0\Big].
    \label{eqn:pdl-sum-ii}
\end{multline}
To bound $\bE^{\pi'}[r(x_k,u_k)+\gamma \cV_0^\pi(X_{k+1},Y_{k+1},Z_{k+1})|H_k,Z_k]$, note that $\{(X_k,Y_k,Z_k,U_k):k\geq 0\}$ forms a Markov chain and $b_k$ is sufficient statistics for $X_k$ given $(H_k,Z_k)$. Hence,
\begin{equation*}
    \bE^{\pi'}[r(X_k,U_k)+\gamma \cV_0^\pi(X_{k+1},Y_{k+1},Z_{k+1})|H_k,Z_k] = \sum_{x_k,u_k}b_k(x_k,H_k)\pi'(u_k|I_k)\cQ_0^\pi(x_k,y_k,Z_k,u_k),
\end{equation*}
where $I_k=(y_k,Z_k)$, and
\begin{align*}
    \begin{aligned}
        \cQ_0^\pi(X_0,Y_0,Z_0,U_0) &= \bE^\pi\Big[\sum_{k=0}^\infty\gamma^kr(X_k,U_k)|X_0,Y_0,Z_0,U_0\Big],\\
        &= \bE^\pi\Big[r(X_0,U_0) + \gamma \cV_0^\pi(X_1, Y_{1}, Z_{1})\Big|X_0,Y_0,Z_0,U_0\Big].
    \end{aligned}
\end{align*}
Also, by the definition of the Bayes-filtered Q-function, we have
\begin{equation*}
    \cQ^\pi(u,y,z) = \sum_{x\in\bX}b_0(x,(y,z))\cQ_0^\pi(x,u,y,z),
\end{equation*}
where $b_0$ is the conditional probability distribution of $X_0$ given $(Y_0,Z_0)$. Hence, we obtain:
\begin{multline}
    \bE^{\pi'}[r_k+\gamma \cV_0^\pi(X_{k+1},Y_{k+1},Z_{k+1})|H_k=h_k,Z_k=z_k] = \sum_{u_k\in\bU}\pi'(u_k|I_k)\cQ^\pi(u_k,y_k,z_k)\\+\sum_{x_k,u_k}\pi'(u_k|I_k)\Big(b_k(x_k,h_k)-b_0(x_k,I_k)\Big)\cQ_0^\pi(x_{k},u_k,y_{k},z_{k}),
    \label{eqn:pdl-sum-ii.5}
\end{multline}
where we used the last two identity in the expansion. Note that $\Big|\cQ_0^\pi(x,u,y,z)\Big|\leq \frac{2}{1-\gamma}$ and the factor
$b_k(x_k,h_k)-b_0(x_k,I_k)$ does not depend on $u_k$. Thus, we have
\begin{multline}
    \bE^{\pi'}[r_k+\gamma \cV_0^\pi(X_{k+1},Y_{k+1},Z_{k+1})|H_k=h_k,Z_k=z_k] \\ \geq  \sum_{u_k\in\bU}\pi'(u_k|I_k)\cQ^\pi(u_k,y_k,z_k)-\frac{1}{1-\gamma}\|b_k(\cdot,h_k)-b_0(\cdot,I_k)\big\|_{\rm TV},
    \label{eqn:pdl-sum-iii}
\end{multline}
since $\|\mu-\nu\|_1 = 2\|\mu-\nu\|_{\rm TV}$ for any $\mu,\nu\in\Sigma(\bX)$. Substituting \eqref{eqn:pdl-sum-iii} into \eqref{eqn:pdl-sum-ii}, we obtain:
\begin{multline}\bE^{\pi'}[r(X_k,U_k)+\gamma \cV_0^\pi(X_{k+1},Y_{k+1},Z_{k+1})-\cV^\pi(Y_k,Z_k)|Y_0=y_0,Z_0=z_0] \\ \geq \bE^{\pi'}\Big[\cA^\pi(U_k,Y_k,Z_k)-\frac{1}{1-\gamma}\|b_k(\cdot,h_k)-b_0(\cdot,I_k)\|_{\rm TV}\Big|Y_0=y_0,Z_0=z_0\Big],
\label{eqn:pdl-sum-iv}
\end{multline}
where $\cA^\pi(u,y,z)=\cQ^\pi(u,y,z)-\cV^\pi(y,z)$. By using the dominated convergence theorem again on \eqref{eqn:pdl-sum-iv}, we conclude that
\begin{multline}
    \bE^{\pi'}\Big[\sum_{k=0}^\infty\gamma^k\Big(r(X_k,U_k)+\gamma \cV_0^\pi(X_{k+1},Y_{k+1},Z_{k+1})-\cV^\pi(Y_k,Z_k)\Big)\Big|Y_0=y_0,Z_0=z_0\Big] \\ \geq \bE^{\pi'}\Big[\sum_{k=0}^\infty \cA^\pi(U_k,Y_k,Z_k)\Big|Y_0=y_0,Z_0=z_0\Big]-\frac{1}{1-\gamma}\cdot \epsilon_{\rm inf}^{\pi'}(y_0,z_0).
\end{multline}

\textbf{Bounding $(ii)$ in \eqref{eqn:pdl-sum-i}:} By following identical steps as in \eqref{eqn:pdl-sum-ii.5}, we get
\begin{equation*}
    \bE^{\pi'}\Big[\sum_{k=0}^\infty\gamma^{k+1}\Big(\cV^\pi(Y_{k+1},Z_{k+1})-\cV_0^\pi(X_{k+1},Y_{k+1},Z_{k+1})\Big|Y_0=y_0,Z_0=z_0\Big] \geq -\frac{1}{1-\gamma}\cdot \epsilon_{\rm inf}^{\pi'}(y_0,z_0).
\end{equation*}
Hence, we conclude that $$\cV^{\pi'}(y_0,z_0)-\cV^{\pi}(y_0,z_0) \geq \bE^{\pi'}\Big[\sum_{k=0}^\infty \cA^\pi(U_k,Y_k,Z_k)\Big|Y_0=y_0,Z_0=z_0\Big]-\frac{2}{1-\gamma}\cdot \epsilon_{\rm inf}^{\pi'}(y_0,z_0).$$ 
\end{proof}

\subsection{Proof of Theorem \ref{thm:nac-fsc}}
\begin{proof}[Proof of Theorem \ref{thm:nac-fsc}]
    The first part of the proof is based on a Lyapunov drift result, which is an extension of the analysis provided in \cite{agarwal2021theory} for natural policy gradient for (fully observable) MDPs. For $\pi\in\Pi_A$, let
    \begin{align*}
            \Lambda(\pi) = \sum_{y\in\bY}\sum_{z\in\bZ}\mathbf{d}_\xi^{\pi^*}(y,z)\mathcal{D}_{\rm KL}(\pi^*(\cdot|y,z)||\pi(\cdot|y,z))
    \end{align*}
    be the potential function, where $\mathbf{d}_\xi^\pi$ is the discounted action-observation visitation distribution under $\pi$. For any $t\geq 0$, we have the following drift:
    \begin{equation}
        \Lambda(\pi_{t+1})-\Lambda(\pi_t) = \sum_{y\in\bY}\sum_{z\in\bZ}\mathbf{d}_\xi^{\pi^*}(y,z)\sum_{u\in\bU}\pi^*(u|y,z)\log\frac{\pi_t(u|y,z)}{\pi_{t+1}(u|y,z)}.
    \end{equation}
    Note $\log\pi_\theta(u|y,z)$ is $1$-smooth with $\sup_{u,z,y}\|\psi(u,y,z)\|_2\leq 1$ \cite{agarwal2021theory}. Thus,
    \begin{equation}
        |\log\pi_{\theta'}(u|y,z)-\log\pi_\theta(u|y,z)-\langle\nabla\log\pi_\theta(u|y,z),\theta'-\theta\rangle| \leq \|\theta-\theta'\|_2^2,
    \end{equation}
    for any $u,y,z$ and $\theta,\theta'\in\bR^d$, which implies that:
    \begin{equation}
        \log\frac{\pi_t(u|y,z)}{\pi_{t+1}(u|y,z)}\leq \eta^2\|\bar{w}_t\|_2^2-\eta\langle\nabla\log\pi_t(u|y,z),\bar{w}_t\rangle,
    \end{equation}
    where $\bar{w}_t = \frac{1}{N}\sum_{k<N}w_t(k).$ Hence,
    \begin{equation}
        \Lambda(\pi_{t+1})-\Lambda(\pi_t) \leq -\eta\sum_{y\in\bY}\sum_{u\in\bU}\sum_{z\in\bZ}\mathbf{d}_\xi^{\pi^*}(y,z)\pi^*(u|y,z)\langle\nabla\log\pi_t(u|y,z),\bar{w}_t\rangle + \eta^2\|\bar{w}_t\|_2^2,
    \end{equation}
    which leads to: \begin{multline}
        \Lambda(\pi_{t+1})-\Lambda(\pi_t) \leq \eta^2R^2 - \eta\sum_{y\in\bY}\sum_{u\in\bU}\sum_{z\in\bZ}\mathbf{d}_\xi^{\pi^*}(y,z)\pi^*(u|y,z)\cA^{\pi_t}(u,y,z)\\+\eta\sqrt{\sum_{y\in\bY}\sum_{u\in\bU}\sum_{z\in\bZ}\mathbf{d}_\xi^{\pi^*}(y,z)\pi^*(u|y,z)\Big(\langle\nabla\log\pi_t(u|y,z),\bar{w}_t\rangle-\cA^{\pi_t}(u,y,z)\Big)^2},
        \label{eqn:drift-inequlity}
    \end{multline} For any $t<T$ and $w\in\bR^d$, let
\begin{align}
\begin{aligned}
        L_{0,t}(w) &= \bE\big[\big(\nabla^\top\log\pi_t(U|Y,Z)w-\cA^{\pi_t}(U,Y,Z)\big)^2\big|\mathfrak{H}_{t}\big],\\
    \widehat{L}_{0,t}(w) &= \bE\big[\big(\nabla^\top\log\pi_t(U|Y,Z)w-\widehat{\cA}_K^{\pi_t}(U,Y,Z)\big)^2\big|\mathfrak{H}_{t}\big],
    \end{aligned}
\end{align}
where $\mathfrak{H}_t$ is the $\sigma$-field generated by all samples used in policy optimization steps (up to and excluding $t$) and policy evaluation step at iteration $t$. By Theorem \ref{thm:m-tdl} and Jensen's inequality,
\begin{equation*}
    \bE\big[\big(\widehat{\cA}_K^{\pi_t}(U,Y,Z)-\cA^{\pi_t}(U,Y,Z)\big)^2\big|\mathfrak{H}_t'\big] \leq \epsilon_{\rm critic}(t),
\end{equation*}
where $\mathfrak{H}_t'$ is the $\sigma$-field generated by all variables in the policy optimization steps before $t$, and $$\epsilon_{\rm critic}(t) = 2\Big( \sqrt{\frac{\|\beta_0-\beta_\pi\|_2^2+M^2(\gamma,R)}{K^{1/2}(1-\gamma^m)}}+\frac{\epsilon_{\rm app}(R)}{1-\gamma^m}+\epsilon_{\rm pa}^{\pi_t}(\gamma, m, R)\Big).$$ Thus, by using the inequality $(x+y)^2 \leq 2x^2 + 2y^2,~x,y\in\bR$, we have:
\begin{equation}\min_w~\widehat{L}_{0,t}(w) \leq \min_w~2L_{0,t}(w)+2\epsilon_{\rm critic}(t).\label{eqn:error-analysis-a}\end{equation}
By Theorem 14.8 in \cite{shalev2014understanding}, the SGD iterations with the step-size choice $\zeta$ yield the following:
\begin{equation}
    \widehat{L}_{0,t}(\bar{w}_t) \leq \epsilon_{\rm actor} + \min_w \widehat{L}_{0,t}(w),
    \label{eqn:stat-error-2}
\end{equation}
where $\epsilon_{actor} = \frac{2-\gamma}{1-\gamma}\cdot\frac{R}{\sqrt{N}}$
at each iteration $t\leq T$. Similarly,
${L}_{0,t}(\bar{w}_t) \leq 2\widehat{L}_{0,t}(\bar{w}_t)+2\epsilon_{\rm critic}(t).$
Thus, taking expectation over the samples, we obtain:
\begin{align*}
    \bE[L_{0,t}(\bar{w}_t)] &\leq 2\bE[\widehat{L}_{0,t}(w_t)] + 2\epsilon_{\rm critic}(t) \leq 2\min_w\widehat{L}_{0,t}(w)+2\epsilon_{\rm actor}+2\epsilon_{\rm critic}(t),\\
    &\leq 4\min_w L_{0,t}(w) + 2\epsilon_{\rm actor} + 6\epsilon_{\rm critic}(t),
\end{align*}
where the second line follows from the definition of $\epsilon_{actor}$ and the last line follows from \eqref{eqn:error-analysis-a}. Since $\min_w L_{0,t}(w) \leq 2\epsilon_{\rm app}(R),$ we have
$\sqrt{\bE[L_{0,t}(\bar{w}_t)]} \leq 4\Big(\epsilon_{\rm app}(R)+\epsilon_{\rm actor}+\epsilon_{\rm critic}(t)\Big).$ By taking expectation of the drift inequality \eqref{eqn:drift-inequlity}, using the above inequality and Lemma \ref{lemma:pdl},
\begin{equation*}
        \bE[\Lambda(\pi_{t+1})-\Lambda(\pi_t)] \leq 4\eta\bar{C}_\infty\Big(\epsilon_{\rm app}(R)+\epsilon_{\rm actor}+\bE\epsilon_{\rm critic}(t)\Big)\\+\eta^2R^2 -(1-\gamma)\eta\Big(\bE\Delta_t - 2\epsilon_{\rm inf}^{\pi^*}(\xi)\Big),
    \end{equation*}
    where $\Delta_t = \cV^{\pi^*}(\xi)-\cV^{\pi_t}(\xi).$ The proof then follows by telescoping sum over $t < T$, re-arranging the terms, and using the step-size choice in the theorem statement. 
\end{proof}

\section{Memory-Inference Error Tradeoff: Proof of Proposition \ref{prop:filter-stability}}\label{app:filter-stability}

\begin{proposition}[Memory-performance tradeoff]\normalfont
    Under Condition \ref{cond:minorization}, for any $n \geq 1$, we have
       $ \epsilon_{\rm inf}^{\pi^*}(\xi) \leq \frac{1}{(1-\gamma)}\cdot \cO\Big(\big(1-\epsilon_0^2\big)^{\lfloor\frac{n}{m_0}\rfloor}\Big)$.
    \label{prop:filter-stability-alt}
\end{proposition}

\subsection{Proof of Proposition \ref{prop:filter-stability}}
The proof will follow a similar strategy described in \cite{van2008hidden}. The main specific challenge in our case is incorporating the control actions into the filter stability results. For hidden Markov chains (HMCs) considered in \cite{van2008hidden}, due to the discrete memoryless observation channel $\Phi$, the observation $Y_k$ depends only on $X_k$. On the other hand, in the case of POMDPs that we consider here, the controller interacts with the environment, and the data obtained from the environment is $(Y_k,U_{k-1})$, where $U_k$ partially depends on the observation history. This necessitates different conditions to establish filter stability, which we establish in this section. We begin with an important lemma.

\begin{definition}
    Let $\upsilon\in\Sigma(\bX)$ be a probability measure, and $\cK(\cdot|x)\in\Sigma(\bX)$ be a transition kernel. Then, we define $\circledast$ as
        $(\cK\circledast\upsilon)(x) = \sum_{x'\in\bX}\upsilon(x')\cK(x|x')$ for any $x\in\bX$.
        \label{def:left-mult}
\end{definition}
\begin{lemma}[Lemma 5.2 in \cite{van2008hidden}]
    Let $\upsilon,\upsilon'\in\Sigma(\bX)$ be two probability mass functions on $\bX$, and $\{\cK(\cdot|x)\in\Sigma(\bX):x\in\bX\}$ be a transition kernel. Then,
    \begin{enumerate}
        \item The operator $\cK\circledast$ is $1$-Lipschitz with respect to the total-variation distance $\|\cdot\|_{\rm TV}$: $$\|\cK\circledast\upsilon-\cK\circledast\upsilon'\|_{\rm TV} \leq \|\upsilon-\upsilon'\|_{\rm TV},$$
        \item (Minorization) If there exist $\mu\in\Sigma(\bX)$ and $\epsilon_0\in(0,1)$ such that $$\cK(x|x') \geq \epsilon_0\cdot \mu(x),~~\forall x,x'\in\bX,$$
        then we have a contraction:
        \begin{equation}
            \|\cK\circledast\upsilon-\cK\circledast\upsilon'\|_{\rm TV} \leq (1-\epsilon_0)\cdot \|\upsilon-\upsilon'\|_{\rm TV}.
        \end{equation}
    \end{enumerate}
    \label{lemma:minorization}
\end{lemma}

\textbf{Main idea:} We can show that, for any $k < n$ and $H_n=(H_0,Y_1^n,U^{n-1})$, the stochastic process $\{X_k:k\geq 0\}$ has a conditional Markovianity property: $$\bP(X_{k+1}=x_{k+1}|X^k,H_n) = \bP(X_{k+1}=x_{k+1}|X_k,H_n),~~\forall x_{k+1},~~w.p.~1.$$ Hence, we can express the $n$-step filtering transformation \eqref{eqn:filtering} as
\begin{equation}
    \bP(X_n=\cdot|H_n) = \cK_{\lambda-1|n}\circledast\Big(\ldots\cK_{1|n}\circledast\Big(\cK_{0|n}\circledast\bP(X_0=\cdot|H_n)\Big)\Big)\Big),
\end{equation}
where $n = \lambda m_0$ and $\cK_{\ell|n}(x_{(\ell+1)m_0}|x_{\ell m_0}) = \bP\big(X_{(\ell+1)m_0}=x_{(\ell+1)m_0}|X_{\ell m_0},H_n\big).$ Now, for every $\ell = 0,1,\ldots,\lambda-1$, if the transition kernel $\cK_{\ell|n}$ satisfies the minorization condition in Lemma \ref{lemma:minorization} for fixed $\mu_{\ell|n}\in\Sigma(\bX)$ and $\epsilon_0 \in (0,1)$, then Lemma \ref{lemma:minorization} implies that $\bP(X_n=\cdot|H_n)$ for two different prior distributions for $X_0$ converges to the same distribution in total-variation distance at a geometric rate with exponent $\lambda = n/m_0$. The key part of the proof is to show that Conditions \ref{cond:poe}-\ref{cond:minorization} suffice to minorize $\cK_{\ell|n}$ for all $\ell\in[0,\lambda)$.

    \textbf{Backward variable.} For any $k<n$, let 
    \begin{equation}
        \beta_{k|n}(x_k,h_k;y_{k+1}^n,u_k^{n-1}) = \bP(Y_{k+1}^n=y_{k+1}^n,U_k^{n-1}=u_k^{n-1}|X_k=x_k,H_k=h_k).
        \label{def:b-variable}
    \end{equation}
Notably, it is straightforward to show that the backward variable $\beta_{k|n}$ satisfies the recursion
\begin{multline}
    \beta_{k|n}(x_k,h_k;y_{k+1}^n,u_k^{n-1})\\ = \sum_{x_{k+1}\in\bX}\pi^*(u_k|y_k,z_k)\cP(x_{k+1}|x_k,u_k)\Phi(y_{k+1}|x_{k+1})\beta_{k+1|n}(x_{k+1},h_{k+1}; y_{k+2}^n,u_{k+1}^{n-1}),
\end{multline}
with $\beta_{n|n} = 1$. As such, $\beta_{k|n}$ is $\sigma(X_k,H_k)$-measurable, does not depend on $X^{k-1}$ or $b_0(\cdot)$.

\begin{lemma}[Conditional Markovianity under an FSC]\normalfont
    For any $k < n$, we have
    \begin{align*}\bP(X_{k+1}=x_{k+1}&|X^k=x^k,H_n=h_n) = \bP(X_{k+1}=x_{k+1}|X_k=x_k,H_n=h_n),\\
        &\propto {\cP(x_{k+1}|x_k,u_k)\Phi(y_{k+1}|x_{k+1})\beta_{k+1|n}(x_{k+1},{h}_{k+1};y_{k+2}^n,u_{k+1}^n)}= \tilde{\kappa}_{k|n}(x_{k+1}|x_k).
    \end{align*}
    \label{lemma:conditional-markovianity}
\end{lemma} Based on Lemma \ref{lemma:conditional-markovianity}, for any $m_0 \geq 1$, we can establish a conditional version of the Chapman-Kolmogorov equation for POMDPs:
\begin{align*}
    \bP(X_{(\ell+1) m_0}&=x_{(\ell+1) m_0}|X_{\ell m_0}=x_{\ell m_0},H_n=h_n) \\&= \tilde{\kappa}_{\ell m_0|n}\circledast(\tilde{\kappa}_{\ell m_0+1|n}\circledast(\ldots \circledast(\tilde{\kappa}_{(\ell+1) m_0|n}(x_{(\ell+1)m_0}|x_{\ell m_0}))\ldots))
    \\&=: \kappa_{\ell|n}^{m_0}(x_{(\ell+1)m_0}|x_{\ell m_0}),
\end{align*}
For a given (potentially random) prior $v_0 = \bP(X_0=\cdot|H_0)\in\Sigma(\bX)$, for $n = \lambda m_0$, we have
\begin{equation}
    \bP_{v_0}(X_n = \cdot|H_n) = \kappa_{\lambda-1|n}^{m_0}\circledast(\ldots\kappa_{1|n}^{m_0}\circledast(\kappa_{0|n}^{m_0}\circledast\phi(v_0,h_n))\ldots),
\end{equation}
where
    $\phi(v_0,h_n)(x) = \bP(X_0 = x|H_n=h_n)
    = \frac{v_0(x)\cdot\beta_{0|n}(x,h_0;y_1^n,u^{n-1})}{\sum_{x'\in\bX}v_0(x')\cdot\beta_{0|n}(x',h_0;y_1^n,u^{n-1})}$
is the posterior. Using this, given $v_0,v_0'\in\Sigma(\bX)$, we want to bound $\|\bP_{v'_0}(X_n=\cdot|H_n)-\bP_{v_0}(X_n=\cdot|H_n)\|_{\rm TV}.$ In the following, we show that Conditions \ref{cond:poe}-\ref{cond:minorization} lead to the minorization of the $m_0$-step transition kernels $\kappa_{\ell|n}^{m_0}$ for $m_0\geq 1$ specified in Condition \ref{cond:minorization}.

\begin{lemma}[Minorization of the smoothing kernel]
    Under Conditions \ref{cond:poe}-\ref{cond:minorization}, there exist $\epsilon_0 \in (0,1)$ and a probability measure $\nu_{\ell|n}\in\Sigma(\bX)$ for any $\ell\in(0,\lambda)$ such that the following holds:
    \begin{equation}
        \kappa_{\ell|n}^{m_0}(x_{(\ell+1)m_0}|x_{\ell m_0}) \geq \alpha^{2m_0-2}\cdot\epsilon_0^2\cdot\nu_{\ell|n}(x_{(\ell+1) m_0}),
    \end{equation}
    for all $x_{\ell m_0},x_{(\ell+1)m_0}\in\Sigma(\bX)$ given $h_n\in\bH\times\bY^n\times\bU^n$.
    \label{lemma:smoothing-kernel-minorization}
\end{lemma}
\begin{proof}
    First, notice that we have the following:
    \begin{multline}
        \bP(X_{(\ell+1)m_0}=x_{(\ell+1)m_0}|X_{\ell m_0}=x_{\ell m_0},H_n=h_n) \\= \frac{\mathsf{q}\big(x_{(\ell+1)m_0},x_{\ell m_0},y_{\ell m_0+1}^{(\ell+1)m_0},u_{\ell m_0}^{(\ell+1)m_0-1}, y_{(\ell+1)m_0+1}^n,u_{(\ell+1)m_0}^{n-1},h_{\ell m_0}\big)}{\sum_{x_{(\ell+1)m_0}'\in\bX}\mathsf{q}\big(x_{(\ell+1)m_0'},x_{\ell m_0},y_{\ell m_0+1}^{(\ell+1)m_0},u_{\ell m_0}^{(\ell+1)m_0-1}, y_{(\ell+1)m_0+1}^n,u_{(\ell+1)m_0}^{n-1},h_{\ell m_0}\big)},
        \label{eqn:smoothing-kernel}
    \end{multline}
    where we decomposed $h_n = (y_{\ell m_0+1}^{(\ell+1)m_0},u_{\ell m_0}^{(\ell+1)m_0-1}, y_{(\ell+1)m_0+1}^n,u_{(\ell+1)m_0}^{n-1},h_{\ell m_0})$ and $\mathsf{q}$ is the joint distribution of $(X_{(\ell+1)m_0},X_{\ell m_0},H_n)$. We can expand the numerator of \eqref{eqn:smoothing-kernel} as follows:
    \begin{align*}
        &\mathsf{q}\big(x_{(\ell+1)m_0},x_{\ell m_0},y_{\ell m_0+1}^{(\ell+1)m_0},u_{\ell m_0}^{(\ell+1)m_0-1}, y_{(\ell+1)m_0+1}^n,u_{(\ell+1)m_0}^{n-1},h_{\ell m_0}\big) = \\&\hskip -0.1cm  \times \bP\big(Y_{(\ell+1)m_0+1}^n=y_{(\ell+1)m_0+1}^n,U_{(\ell+1) m_0}^{n-1}=u_{(\ell+1) m_0}^{n-1}\big|X_{(\ell+1)m_0}=x_{(\ell+1)m_0},H_{(\ell+1)m_0}=h_{(\ell+1)m_0}\big)\\ & \hskip -0.1cm \times \bP\big(A_{\ell m_0}\big) \bP(X_{(\ell+1)m_0}=x_{(\ell+1)m_0},Y_{\ell m_0+1}^{(\ell+1)m_0}=y_{\ell m_0+1}^{(\ell+1)m_0},U_{\ell m_0}^{(\ell+1)m_0-1}=u_{\ell m_0}^{(\ell+1)m_0-1}|A_{\ell m_0}),
        \label{eqn:kernel-numerator}
    \end{align*}
    where $A_{\ell m_0}=\{X_{\ell m_0}=x_{\ell m_0},H_{\ell m_0}=h_{\ell m_0}\}$, and the first term on the RHS of the above identity follows from:
    \begin{multline}
\bP\big(Y_{(\ell+1)m_0+1}^n=y_{(\ell+1)m_0+1}^n,U_{(\ell+1) m_0}^{n-1}=u_{(\ell+1) m_0}^{n-1}\big|X_{(\ell+1)m_0},H_{(\ell+1)m_0}\big) \\= \bP\big(Y_{(\ell+1)m_0+1}^n=y_{(\ell+1)m_0+1}^n,U_{(\ell+1) m_0}^{n-1}=u_{(\ell+1) m_0}^{n-1}\big|X_{(\ell+1)m_0},X_{\ell m_0},H_{(\ell+1)m_0}\big).
        \label{eqn:markovianity}
    \end{multline}
    The above identity is true with probability 1 since (i) $U_{(\ell+1)m_0}$ is $\sigma(H_{(\ell+1)m_0})$-measurable, (ii) $\{X_k:k\geq 0\}$ is a controlled Markov chain, (iii) $\Phi$ is a discrete memoryless channel. From Definition \ref{def:b-variable}, for $A_{(\ell+1)m_0}=\{(X_{(\ell+1)m_0},H_{(\ell+1)m_0})=(x_{(\ell+1)m_0},h_{(\ell+1)m_0})\}$, we observe that 
    \begin{multline*}    \bP\Big(\big(Y_{(\ell+1)m_0+1}^n,U_{(\ell+1) m_0}^{n-1}\big)=\big(y_{(\ell+1)m_0+1}^n,u_{(\ell+1) m_0}^{n-1}\big)\Big|A_{(\ell+1)m_0}\Big) \\= \beta_{(\ell+1)m_0|n}(x_{(\ell+1)m_0}, h_{(\ell+1)m_0}; y_{(\ell+1)m_0+1}^n,u_{(\ell+1) m_0}^{n-1}).
    \end{multline*}
    Thus, substituting the above identity and \eqref{eqn:kernel-numerator} into \eqref{eqn:smoothing-kernel}, we obtain the following:
    \begin{multline}
        \kappa_{\ell|n}^{m_0}(x_{(\ell+1)m_0}|x_{\ell m_0}) =\frac{\bP\big(Y_{(\ell+1)m_0+1}^n=y_{(\ell+1)m_0+1}^n,U_{(\ell+1) m_0}^{n-1}=u_{(\ell+1) m_0}^{n-1}\big|A_{(\ell+1)m_0}\big)}{W(x_{\ell m_0}, h_n)}\\\times\bP(X_{(\ell+1)m_0}=x_{(\ell+1)m_0},Y_{\ell m_0+1}^{(\ell+1)m_0}=y_{\ell m_0+1}^{(\ell+1)m_0},U_{\ell m_0}^{(\ell+1)m_0-1}=u_{\ell m_0}^{(\ell+1)m_0-1}|A_{\ell m_0}),
        \label{eqn:kernel-min-1}
    \end{multline} where \begin{multline*}
        W(x_{\ell m_0}, h_n) = \sum_{x'}\Big[\bP(X_{(\ell+1)m_0}=x',Y_{\ell m_0+1}^{(\ell+1)m_0}=y_{\ell m_0+1}^{(\ell+1)m_0},U_{\ell m_0}^{(\ell+1)m_0-1}=u_{\ell m_0}^{(\ell+1)m_0-1}|A_{\ell m_0})\\ \times \bP\big((Y_{(\ell+1)m_0+1}^n,U_{(\ell+1) m_0}^{n-1})=(y_{(\ell+1)m_0+1}^n,u_{(\ell+1) m_0}^{n-1})\big|(X_{(\ell+1)m_0},H_{(\ell+1)m_0})=(x',h_{(\ell+1)m_0})\big)\Big].
    \end{multline*} Now, we will use Conditions \ref{cond:poe}-\ref{cond:minorization} to show that the conditional probability \begin{multline}\widetilde{\bP}(x_{(\ell+1)m_0},y_{\ell m_0+1}^{(\ell+1)m_0},u_{\ell m_0}^{(\ell+1)m_0-1}|x_{\ell m_0},h_{\ell m_0})\\:=\bP(X_{(\ell+1)m_0}=x_{(\ell+1)m_0},Y_{\ell m_0+1}^{(\ell+1)m_0}=y_{\ell m_0+1}^{(\ell+1)m_0},U_{\ell m_0}^{(\ell+1)m_0-1}=u_{\ell m_0}^{(\ell+1)m_0-1}|X_{\ell m_0},H_{\ell m_0}),\end{multline} minorizes and majorizes simultaneously with probability 1, which will let us show that $\kappa_{\ell|n}^{m_0}$ minorizes. First, note that we can perform the following expansion:
    \begin{multline*}
        \widetilde{\bP}(x_{(\ell+1)m_0},y_{\ell m_0+1}^{(\ell+1)m_0},u_{\ell m_0}^{(\ell+1)m_0-1}|x_{\ell m_0},h_{\ell m_0}) \\= \sum_{x_{\ell m_0+1}^{(\ell+1)m_0-1}}\prod_{j=\ell m_0}^{(\ell+1)m_0-1}\pi^*(u_j|y_j,z_j)\cP(x_{j+1}|x_j,u_j)\Phi(y_{j+1}|x_{j+1}).
    \end{multline*}
    For any probability measure $\mu\in\Sigma(\bU)$, let
    \begin{equation}
        \bP^{{\mu}}(x_{m_0},y_1^{m_0}, u^{m_0-1}|x_0) = \sum_{x_1^{m_0-1}}\prod_{j=0}^{m_0-1}{\mu}(u_j)\cP(x_{j+1}|x_j,u_j)\Phi(y_{j+1}|x_{j+1}).
    \end{equation}
    Then, under Condition \ref{cond:poe}, we have the following inequalities:
    \begin{align}
        \nonumber \alpha^{m_0-1}\bP^{\bar{\mu}}(x_{(\ell+1)m_0},y_{\ell m_0+1}^{(\ell+1)m_0},u_{\ell m_0}^{(\ell+1)m_0-1}|&x_{\ell m_0}) \leq \widetilde{\bP}(x_{(\ell+1)m_0},y_{\ell m_0+1}^{(\ell+1)m_0},u_{\ell m_0}^{(\ell+1)m_0-1}|x_{\ell m_0},h_{\ell m_0}),\\
        &\leq \bP^{\bar{\mu}}(x_{(\ell+1)m_0},y_{\ell m_0+1}^{(\ell+1)m_0},u_{\ell m_0}^{(\ell+1)m_0-1}|x_{\ell m_0})/\alpha^{m_0-1} \label{eqn:m-1}.
    \end{align} Furthermore, Condition \ref{cond:minorization} implies that
    \begin{align}
    \begin{aligned}
        \epsilon_0\cdot\nu\big(x_{(\ell+1)m_0},y_{\ell m_0+1}^{(\ell+1)m_0},u_{\ell m_0}^{(\ell+1)m_0-1}\big) &\leq \bP^{\bar{\mu}}(x_{(\ell+1)m_0},y_{\ell m_0+1}^{(\ell+1)m_0},u_{\ell m_0}^{(\ell+1)m_0-1}|x_{\ell m_0}),\\& \leq \nu\big(x_{(\ell+1)m_0},y_{\ell m_0+1}^{(\ell+1)m_0},u_{\ell m_0}^{(\ell+1)m_0-1}\big)/\epsilon_0.
        \end{aligned}
        \label{eqn:m-2}
    \end{align} Combining \eqref{eqn:m-1} and \eqref{eqn:m-2}, we obtain simultaneous minorization-majorization
    \begin{align}
        \begin{aligned}
            \epsilon_0\alpha^{m_0-1}\nu\big(x_{(\ell+1)m_0},y_{\ell m_0+1}^{(\ell+1)m_0},u_{\ell m_0}^{(\ell+1)m_0-1}\big) &\leq \widetilde{\bP}(x_{(\ell+1)m_0},y_{\ell m_0+1}^{(\ell+1)m_0},u_{\ell m_0}^{(\ell+1)m_0-1}|x_{\ell m_0},h_{\ell m_0}),\\
        &\leq \frac{1}{\epsilon_0\alpha^{m_0-1}}\nu\big(x_{(\ell+1)m_0},y_{\ell m_0+1}^{(\ell+1)m_0},u_{\ell m_0}^{(\ell+1)m_0-1}\big),
        \end{aligned}
        \label{eqn:m-3}
    \end{align} Using the lower bound and upper bound in \eqref{eqn:m-3}, we obtain the following bound for \eqref{eqn:kernel-min-1}:
    \begin{equation*}
        \kappa_{\ell|n}^{m_0}(x_{(\ell+1)m_0}|x_{\ell m_0}) \geq (\epsilon_0\alpha^{m_0-1})^2\cdot \nu_{\ell|n}^{m_0}(x_{(\ell+1)m_0}),\mbox{where }
    \end{equation*}
    \begin{equation*}
        \nu_{\ell|n}^{m_0}(x_{(\ell+1)m_0}) \propto \nu(x_{(\ell+1)m_0},y_{\ell m_0+1}^{(\ell+1)m_0},u_{\ell m_0}^{(\ell+1)m_0-1})\cdot \bP\big(y_{(\ell+1)m_0+1}^n,u_{(\ell+1) m_0}^{n-1}\big|x_{(\ell+1)m_0},h_{(\ell+1)m_0}\big).
    \end{equation*}
    From the discussion in \eqref{eqn:markovianity}, we deduce that $\nu_{\ell|n}^{m_0}(x_{(\ell+1)m_0})$ does \emph{not} depend on $x_{\ell m_0}$.
\end{proof}

\begin{proof}[Proof of Proposition \ref{prop:filter-stability}]
    Let $\lambda = \lfloor n/m_0 \rfloor$. For a given prior distribution $v_0\in\Sigma(\bX)$, recall the definition of the posterior:
    \begin{equation}
        \phi(v_0,h_n)(x) = \frac{v_0(x)\cdot\beta_{0|n}(x,h_0;y_1^n,u^{n-1})}{\sum_{x'\in\bX}v_0(x')\cdot\beta_{0|n}(x',h_0;y_1^n,u^{n-1})}.
    \end{equation}
    By using Definition \ref{def:left-mult}, we can express:
    \begin{equation}
        F^{(n)}(v_0, y_1^n, u_0^{n-1}) = \kappa_{\lambda-1|n}^{m_0}\circledast\ldots\circledast\kappa_{1|n}^{m_0}\circledast\kappa_{0|n}^{m_0}\circledast\phi(v_0, h_n) ,
    \end{equation}
    where $\kappa_{\ell|n}^{m_0}$ is the smoothing kernel in \eqref{eqn:kernel-min-1}. Thus, for two (potentially random) prior distributions $v_0,v_0'\in\Sigma(\bX)$, we have:
    \begin{align*}
        \|F^{(n)}(v_0, y_1^n,& u_0^{n-1})-F^{(n)}(v_0', y_1^n, u_0^{n-1})\|_{\rm TV} \\&= \|\kappa_{\lambda-1|n}^{m_0}\circledast\ldots\kappa_{0|n}^{m_0}\circledast\phi(v_0, h_n)-\kappa_{\lambda-1|n}^{m_0}\circledast\ldots\kappa_{0|n}^{m_0}\circledast\phi(v_0', h_n) \|_{\rm TV},\\
        &\vdots\\
        &\leq (1-(\epsilon_0\alpha^{m_0-1})^2)^\lambda \|\phi(v_0, h_n)-\phi(v_0', h_n)\|_{\rm TV},
    \end{align*}
    where all inequalities are obtained by successive applications of the contraction result in Lemma \ref{lemma:minorization} and the minorization result for the smoothing kernels $\kappa_{\ell|n}^{m_0}$ in Lemma \ref{lemma:smoothing-kernel-minorization}. Hence, $$\|F^{(n)}(v_0, y_1^n, u_0^{n-1})-F^{(n)}(v_0', y_1^n, u_0^{n-1})\|_{\rm TV} = \cO\Big(\big(1-(\epsilon_0\alpha^{m_0-1})^2\big)^{\lfloor\frac{n}{m_0}\rfloor}\Big).$$
    Note that $v_0,v_0'$ can be random, and they may depend on the history of the decision process. Hence, the above result concludes the proof.
\end{proof}
\bibliographystyle{siamplain}
\bibliography{references}

\end{document}